%% file: colt2026-sample.tex
\PassOptionsToPackage{cleveref}{jmlrutils}
\PassOptionsToPackage{table,dvipsnames}{xcolor}
\documentclass[final]{colt2026} 

\title[]{Tight Sample Complexity of Transformers}
\usepackage{times}
\usepackage{cleveref}
\crefname{theorem}{Theorem}{Theorems}
\Crefname{theorem}{Theorem}{Theorems}
\crefname{lemma}{Lemma}{Lemmas}
\Crefname{lemma}{Lemma}{Lemmas}
\crefname{corollary}{Corollary}{Corollaries}
\Crefname{corollary}{Corollary}{Corollaries}
\crefname{definition}{Definition}{Definitions}
\Crefname{definition}{Definition}{Definitions}
\usepackage{xcolor}
\usepackage{tikz}
\usepackage{booktabs}
\usetikzlibrary{arrows.meta,positioning,calc,decorations.pathreplacing,matrix}

\newcommand{\fixed}[1]{#1}

\newcommand{\removed}[1]{}

\newcommand{\sCoT}{\mathrm{CoT}}
\newcommand{\sete}{\mathrm{e2e}}

\newcommand{\TSdim}{\mathrm{TSdim}}
\newcommand{\ASdim}{\mathrm{ASdim}}
\newcommand{\Ffam}{\mathfrak{F}}
\newcommand{\concat}{\mathbin{\Vert}}
\newcommand{\Tin}{T}
\newcommand{\Tout}{T'}
\newcommand{\Ttotal}{T_{\mathrm{ctx}}}

\coltauthor{%
 \Name{Chenxiao Yang} \Email{chenxiao@ttic.edu}\\
 \Name{Nathan Srebro} \Email{nati@ttic.edu}\\
 \Name{Zhiyuan Li} \Email{zhiyuanli@ttic.edu}\\
 \addr Toyota Technological Institute at Chicago (TTIC)%
}

\begin{document}

\maketitle

\begin{abstract}
We tightly characterize the VC dimension of depth-$L$ Transformers with a total of $W$ parameters, mapping an input sequence of length $T$ to a single output, establishing an upper bound of $O(L W \log (T W))$ and a nearly matching lower bound of $\Omega(L W \log (T W / L))$. We further tightly characterize the sample complexity of chain-of-thought learning using such a Transformer, showing teacher forcing (i.e. selecting a predictor consistent with the entire chain-of-thought on training data) learns with sample complexity $O\left(L W \log \left(\left(T+T^{\prime}\right) W\right)\right)$ and that any learning rule that uses chain-of-thought data requires at least $\Omega\left(L W \log \left(\left(T+T^{\prime}\right) W / L\right)\right)$ examples, where $T$ is the input length and $T^{\prime}$ is the number of autoregressive steps.
\end{abstract}


\section{Introduction}

Transformers~\citep{vaswani2017attention} are a class of models that operate on sequences of arbitrary length, where the architecture and the number of parameters are independent of the input sequence length $T$.  A significant advantage one might then expect is that the sample complexity of learning is independent of the sequence length. However, we do not yet have a rigorous and precise understanding of this sample complexity, and its dependence on the input sequence length and other aspects of the architecture.  Tightly characterizing the sample complexity is important both for understanding the length dependence (or lack thereof), as well as for translating results on the representational power of Transformers and their required size \cite[e.g.][]{yun2020are,bhattamishra2020ability, hahn2020theoretical,perez2021attention, wei2021statistically, liu2023transformers,feng2023towards,li2024chain,merrill2023expressive} to actual learning guarantees and understanding the inductive bias of Transformers.  We emphasize that in `learning' here,  we are referring to statistical learning by, e.g.~empirical risk minimization, and are ignoring the important computational issue of finding such a risk minimizer.

Prior work analyzed the scale-sensitive sample complexity of Transformers, i.e.~in terms of norm restrictions on the weights (and under margin assumptions or for scale sensitive loss).  In particular, \citet{trauger2023sequence} obtained sample complexity that depends polynomially on the norms of weight matrices (and thus polynomially on the number of parameters under standard scaling), and exponentially on the depth, but as desired is independent of input sequence length\footnote{In this regard, these improve over a previous bound \citep{edelman2022inductive}, that has better polynomial dependency on the norms, but does depend logarithmically on sequence length}.  Similarly, \cite{zhang2022analysis} also established sequence-length independent sample complexity, again by relying on norm bounds, but with only a logarithmic norm dependence, and a suboptimal quadratic dependence on the number of parameters (times a linear dependence on depth).

In this paper, instead of relying on the magnitudes of the weights and the scale of the outputs, we investigate the parametric capacity of Transformers, i.e.~in terms of the number of weights, allowing the weights to be arbitrary real numbers.  As the VC dimension tightly captures the sample complexity of PAC-learning \cite{vapnik1971uniform,blumer1989learnability}, this corresponds to characterizing the VC dimension (or its multiclass generalization) of the class of functions implementable by Transformers, with any real valued weights. This is a basic question about parametric classes and is fundamental in learning theory, with well-understood answers for many classical classes (e.g., linear predictors, polynomials). Even for understanding scale-sensitive or regularization-driven generalization, this serves as an important baseline.

The VC dimension of feed-forward networks has been extensively studied over a span of fifty years \citep[e.g.][]{cover1968capacity,baum1989size,maass1994neural,koiran1994vc,karpinski1997pfaffianVC,bartlett2019nearly}, and is now well understood.  In particular, we understand that for deep feed-forward networks, the VC dimension can be much larger than the number of parameters and that this greatly depends on the activation function:  with hard-threshold activations there is no depth dependence, and the VC dimension is $\Theta(W \log W)$, where $W$ is the total number of weights. With ReLU (or other piecewise linear) activations, the depth comes in linearly and the VC dimension is $\Omega(W L \log(W/L)) \leq \mathrm{VCdim} \leq \mathcal{O}(WL \log W)$. But with sigmoidal activations it significantly worsens~\citep{koiran1994vc,karpinski1997pfaffianVC}. \textbf{How about the VC dimension of Transformers?  How does it depend on the total number of parameters $W$, the depth $L$, the input sequence length $T$, and perhaps other parameters such as the number of attention heads per layer.}

Based on our experience with feed-forward networks, we would expect Transformer's VC dimension to also depend on the activation functions used. ReLU activations are common for feed-forward layers in practice, and we consider such units here. Attention layer is usually based on a softmax operation. But, a softmax operation hides a logistic activation and can be easily used to simulate it, and so if we allow softmax operations, we encounter the same difficulties and bad scaling as for sigmoidal feed-forward networks discussed above.  Instead, we consider hard-attention, which is commonly studied in the theoretical literature on Transformers \cite[e.g.][]{hahn2020theoretical,merrill2022saturated,merrill2023expressive,yang2025pencil}, and avoids sigmoidal, exponential and logarithmic functions.

\paragraph{Transformer VC-Dimension.}  Accordingly, what we study are Transformers that map $T$ input tokens from some finite-size alphabet $\Sigma$ to a single output token in $\Sigma$ (or an embedding sequence to $\{0,1\}$ as classifiers), and consist of token embedding layer, ReLU feed-forward layers, hard-attention layers (average, leftmost or rightmost hard attention, and any number of attention heads), and greedy decoding. We allow arbitrary, but fixed (i.e.~not learned) additive positional encoding, including no positional encoding (NOPE,~\citet{kazemnejad2023impact}).  We allow skip connections, but no layerwise-normalization.  See precise description in Section \ref{sec:background}.  We show (in \Cref{sec:vc-results}) that for any such Transformers as classifiers, the VC dimension and thus the sample complexity of learning, is $\mathcal O(L W \log(W T))$, where $W$ is the total number of weights, and $L$ is the depth (number of feed-forward and attention layers).  The dependence on $\mathcal O(L W \log W)$ is tight, even just due to the feed-forward sublayers.  We further show that the multiplicative dependence on $\log T$ is also tight by providing a nearly matching lower bound of $\Omega(L W \log(W T/L))$.  \cite{edelman2022inductive} previously establishes a lower bound of $\Omega(\log T$) even when the number of parameters is fixed, but this left open the possibility of an additive $\log T$ term---we improve this to show the necessity of the interaction.

\paragraph{Autoregressive CoT with Transformers.} In \Cref{sec:cot}, we turn to using a Transformer auto-regressively to generate a final answer after a chain-of-thought of length $\Tout$.  Following \citet{joshi2025cot}, we study the sample complexity of learning to map the input of length $\Tin$ to a single-token output obtained after $\Tout$ steps, given training data consisting of $m$ i.i.d.~inputs and their entire $\Tout$-step chain-of-thought. Unlike standard supervised learning where a single measure of VC-dimension governs learnability, we show that two complexity measures govern sample complexity in CoT learning: $\TSdim$ ({trace shattering dimension}), measuring the capacity to shatter entire reasoning traces, governs the upper bound; $\ASdim$ ({answer shattering dimension}), measuring the capacity to shatter only the final answer, governs the CoT lower bound. While combining our single-step VC bound above with the generic result of \citet{joshi2025cot} yields a sample complexity upper bound of $O(W L \log (W (\Tin + \Tout)) \log \Tout/\epsilon)$, directly bounding $\TSdim$ improves it to $\mathcal O(W L \log (W(\Tin + \Tout))/\epsilon)$, avoiding the extra $\log \Tout$ factor. \fixed{For lower bounds, answer shattering gives the CoT term $\Omega(WL\log(\Tin+\Tout)/\epsilon)$; combined with the ordinary supervised ReLU-network lower bound, this implies the worst-case lower bound $\Omega(W L \log((\Tin+\Tout)W/L)/\epsilon)$.} Importantly, the upper bound is obtained by the obvious ``teacher forcing'' \cite{williams1989learning} learning rule, i.e.~selecting any Transformer in the class that is consistent with all the steps in CoT \citep[as in also][]{joshi2025cot}, and the lower bound holds for any learning rule, thus establishing the \fixed{optimality} of teacher forcing for training Transformers with CoT supervision.

Our upper bounds are based on combining proofs about general computation circuits~\citep{bartlett1998almost}, with more recent refinements for piecewise linear neural nets~\citep{bartlett2019nearly}, and then carefully expressing the attention operation using a poly-sized {\em constant-depth} computation.  The lower bounds use novel explicit constructions.

\section{Preliminaries: Transformers and Autoregressive CoT} \label{sec:background}

Throughout this paper, let $\Sigma$ denote a vocabulary with $|\Sigma| = K$. We write token sequences as $X = (x_0, \ldots, x_{T-1}) \in \Sigma^T$ and embedding sequences as $H \in \mathbb{R}^{T \times d}$, where $T$ is the sequence length and $d$ is the embedding dimension. We adopt Python-style indexing: $X[{:}t]$ denotes the first $t$ elements, and $X[{-}t{:}]$ the last $t$ elements; $X\concat Z$ denotes sequence concatenation.

\fixed{We use Transformers in two roles. The fixed-length VC-dimension results use binary classifiers $f^{01}: \mathbb{R}^{T \times d} \to \{0,1\}$. The autoregressive CoT results use next-token generators with input length $\Tin$, generation length $\Tout$, and maximum context length $\Ttotal=\Tin+\Tout$.}

\subsection{Transformers} \label{subsec:Transformers}

We first define the Transformer architecture. Our formulation follows the standard decoder-only Transformer~\citep{vaswani2017attention}, but replaces softmax attention with \emph{hard attention}, a discrete variant that has been central to recent theoretical studies of Transformers~\citep{hahn2020theoretical,merrill2022saturated}. As is standard in VC-dimension analysis, we assume exact arithmetic over the reals. 

\paragraph{Transformer Layers.}
A Transformer of depth $L$ has $L$ Transformer blocks. It transforms an input embedding sequence $H^{(0)} = H \in \mathbb{R}^{T \times d}$ through these blocks. We allow layer-dependent hidden dimensions $d^{(0)},d^{(1)},\dots,d^{(L)}$ with $d^{(0)} = d$, so $H^{(\ell)} \in \mathbb{R}^{T \times d^{(\ell)}}$. Each row $H^{(\ell)}_i \in \mathbb{R}^{d^{(\ell)}}$ represents the embedding of position $i$. Each block applies two sublayers: an attention mechanism followed by a feed-forward network.
\begin{equation}
    \tilde{H}^{(\ell)} = \mathsf{Attn}^{(\ell)}(H^{(\ell-1)}), \quad H^{(\ell)} = \mathsf{FFN}^{(\ell)}(\tilde{H}^{(\ell)}).
\end{equation}
We state attention-before-FFN blocks; residuals and other fixed orderings only change depth by constants.

\paragraph{Hard Attention.}
The key component distinguishing Transformers from feed-forward networks is the attention mechanism, which enables data-dependent routing of information across positions. For each attention head, queries, keys, and values are computed via affine transformations:
\begin{equation}
Q_i = W_Q H_i + b_Q, \quad K_i = W_K H_i + b_K, \quad V_i = W_V H_i + b_V,
\end{equation}
where $W_Q, W_K, W_V$ are projection matrices and $b_Q, b_K, b_V$ are bias vectors (suppressing the layer index $\ell$). We use \emph{causal masking}: position $i$ can only attend to positions $j \leq i$. We analyze \emph{Hard Attention}~\citep{hahn2020theoretical,merrill2022saturated,merrill2023expressive}, the low-temperature limit of softmax attention where weights concentrate on score-maximizing positions. Let $M_i = \{j \leq i : \langle Q_i, K_j \rangle = \max_{k \leq i} \langle Q_i, K_k \rangle\}$ denote the set of causally-visible positions achieving the maximum score, and define attention weights $\alpha_{i,j} = \mathbf{1}[j \in M_i] / |M_i|$. The attention output is $\mathsf{Attn}(H)_i = \sum_{j} \alpha_{i,j} V_j$. Our analysis does not rely on the specific tie-breaking rule and generalizes to variants such as \emph{leftmost-hard} or \emph{rightmost-hard} attention~\citep{hahn2020theoretical}.

\paragraph{Multi-Head Attention.}
A \emph{multi-head attention} layer at layer $\ell$ with $h^{(\ell)}$ heads computes $h^{(\ell)}$ attention outputs in parallel. Each head $i \in \{1, \ldots, h^{(\ell)}\}$ computes queries, keys, and values via its own projection matrices $W_Q^{(\ell,i)}, W_K^{(\ell,i)}, W_V^{(\ell,i)}$, then applies hard attention as defined above to produce $\mathsf{Attn}^{(\ell,i)}(H)$, with value/output dimension matching the FFN input dimension. The outputs are summed: $\mathsf{Attn}^{(\ell)}(H) = \sum_{i=1}^{h^{(\ell)}} \mathsf{Attn}^{(\ell,i)}(H)$.
For simplicity we sum head outputs; the usual concatenate-and-project formulation is covered by allowing head-specific output dimensions and projections.

\paragraph{Feed-Forward Network.}
The $\mathsf{FFN}$ sublayer at layer $\ell$ is a position-wise MLP with ReLU activations, mapping $\mathbb{R}^{d^{(\ell-1)}} \to \mathbb{R}^{d^{(\ell)}}$. \fixed{It has at most $D_0$ internal affine/ReLU layers, for an absolute constant $D_0$, with arbitrary hidden widths collected in $\bar d^{(\ell)}$.} Our results can be extended to any piecewise-linear activation, following~\citet{bartlett1998almost}.

\paragraph{Input Layer.}
For next-token generators operating on token sequences, we first embed tokens into vectors. The \emph{token embedding} $\mathsf{TE}: \Sigma \to \mathbb{R}^{d}$, parametrized by $W_{\mathsf{TE}} \in \mathbb{R}^{d \times K}$, maps each token to a $d$-dimensional vector. \fixed{For a context-length-$T$ architecture, we use a fixed, non-learned positional embedding $\mathsf{PE}: \{0,\ldots,T-1\}\to\mathbb R^d$. In autoregressive results, $T=\Ttotal$, and the same parameters are run on each prefix using the corresponding initial segment of $\mathsf{PE}$.} The input embedding sequence is:
\begin{equation}
    H = \bigl(\mathsf{TE}(x_0) + \mathsf{PE}(0), \ldots, \mathsf{TE}(x_{T-1}) + \mathsf{PE}(T-1)\bigr) \in \mathbb{R}^{T \times d}.
\end{equation}
\fixed{The positional embedding may be arbitrary, but it is fixed and not counted as a trainable parameter.}

\paragraph{Output Layer.}
For a length-$T$ input, the output is extracted from the final position $H^{(L)}_{T-1} \in \mathbb{R}^{d^{(L)}}$ of the last layer. \fixed{For an autoregressive prefix of length $s$, the same rule reads position $s-1$.} For binary classifiers, a linear projection $w_{\mathrm{out}} \in \mathbb{R}^{d^{(L)}}$ followed by a threshold yields the binary label. For next-token generators, a decoding matrix $W_{\mathsf{DE}} \in \mathbb{R}^{K \times d^{(L)}}$ maps the final representation to logits over the vocabulary. \fixed{For a fixed architecture $\mathcal A$ and trainable parameter vector $\theta$,}
\begin{equation}
    \fixed{f^{01}_{\mathcal A,\theta}(H)} = \mathbf{1}\bigl[ w_{\mathrm{out}}^\top H^{(L)}_{T-1} > 0 \bigr] \in \{0, 1\}, \qquad \fixed{f_{\mathcal A,\theta}(X)} = \arg\max_{x \in \Sigma} \bigl[ W_{\mathsf{DE}} H^{(L)}_{T-1} \bigr]_x \in \Sigma.
\end{equation}
The argmax uses a fixed deterministic tie-breaking rule.
In practice, $W_{\mathsf{DE}}$ often shares weights with the token embedding $W_{\mathsf{TE}}$ (\emph{weight tying}); our analysis permits both tied and untied configurations.

\paragraph{Hypothesis Classes.}
\fixed{We distinguish a fixed architecture from a family of possible architectures. A generic Transformer architecture $\mathcal{A}$ is specified by the choices introduced above:}
\begin{equation}
\mathcal{A} = \bigl\{K,\; L,\; \{d^{(i)}\}_{i=0}^L,\; \{h^{(\ell)}\}_{\ell=1}^L,\; \{\bar{d}^{(\ell)}\}_{\ell=1}^L,\; \mathsf{PE},\; \mathsf{Res}\bigr\},
\end{equation}
\fixed{where these entries record the architectural choices above. Here $\mathsf{Res}$ records a residual/no-residual flag for each attention and FFN sublayer. A residual adds the sublayer input to its output; if dimensions differ, the residual path may use an affine projection, counted in $W(\mathcal A)$. Let $W(\mathcal{A})$ and $L(\mathcal{A})$ denote the total parameter count and depth. Varying the trainable parameters gives fixed-architecture hypothesis classes; collecting these over all architectures within a budget gives architecture families:}
\begin{equation}\label{eq:arch-family-defs}
\begin{aligned}
\mathcal{F}_{\mathcal{A}} &:= \{f_{\mathcal A,\theta} : \theta \in \mathbb{R}^{W(\mathcal{A})}\},&
\mathcal{F}^{01}_{\mathcal{A}} &:= \{f^{01}_{\mathcal A,\theta} : \theta \in \mathbb{R}^{W(\mathcal{A})}\},\\[-1mm]
\Ffam_{L,W} &:= \{\mathcal{F}_{\mathcal{A}} : L(\mathcal{A})\le L,\; W(\mathcal{A})\le W\},&
\Ffam^{01}_{L,W} &:= \{\mathcal{F}^{01}_{\mathcal{A}} : L(\mathcal{A})\le L,\; W(\mathcal{A})\le W\}.
\end{aligned}
\end{equation}
\fixed{When $\mathcal A$ is fixed, we often suppress it and write $f_\theta$ or $f^{01}_\theta$.}
\fixed{Thus $\Ffam_{L,W}$ and $\Ffam^{01}_{L,W}$ are families of fixed-architecture hypothesis classes, not single hypothesis classes.}

\subsection{Autoregressive Next-Token Generation} \label{subsec:autoregressive}

\paragraph{Next-Token Generator.}
A \emph{next-token generator} maps a token context to the next token. We write $f:\Sigma^*\to\Sigma$ for convenience, but in our results $f$ is only evaluated on contexts of length at most $\Ttotal$. At each step, the model uses the same $f$ to predict a new token and appends it to the current sequence; we denote this \emph{apply-and-append} operation by $\bar{f}: \Sigma^* \to \Sigma^*$, where $\bar{f}(X) := X \concat f(X)$. Starting from an initial prompt $X \in \Sigma^{\Tin}$ of \emph{input length} $\Tin$, the model iteratively applies $\bar{f}$ to generate $\Tout$ new tokens (the \emph{generation length}):
\begin{equation}
    f^{(\Tout)}(X) := \underbrace{\bar{f} \circ \bar{f} \circ \cdots \circ \bar{f}}_{\Tout \text{ times}}(X) \in \Sigma^{\Tin + \Tout}.
\end{equation}
The generated suffix $f^{(\Tout)}(X)[-\Tout{:}]$ is the CoT trace, and its last token is the final answer.

\paragraph{Teacher Forced Training.}
In practice, autoregressive models are trained using \emph{teacher forcing}~\citep{williams1989learning}: at each step, the model predicts the next token conditioned on the \emph{ground-truth} prefix rather than its own previous predictions. Concretely, consider a distribution $\mathcal{D}$ over input-output pairs $(X, Z)$, where $X \in \Sigma^{\Tin}$ is an input prompt and $Z = (z_0, \ldots, z_{\Tout-1}) \in \Sigma^{\Tout}$ is the target sequence. Given $m$ i.i.d.\ samples $S = \{(X_i, Z_i)\}_{i=1}^m$ from $\mathcal{D}$ and a hypothesis class $\mathcal{F}$ of next-token generators, the learning objective is to find $f \in \mathcal{F}$ minimizing the \emph{empirical next-token prediction loss}:
\begin{equation}\label{eq:ntp-loss}
    \widehat{\mathcal{L}}(f; S) = \frac{1}{m} \sum_{i=1}^{m} \frac{1}{\Tout} \sum_{t=0}^{\Tout-1} \mathbf{1}\bigl[ f(X_i \concat Z_i[{:}t]) \neq Z_i[t] \bigr].
\end{equation}

\paragraph{End-to-End Evaluation.}
At test time, however, we ultimately care only about the correctness of the final token $Z[-1]$, treating the intermediate tokens $Z[{:}{-}1]$ as auxiliary computation. This is captured by the \emph{population end-to-end loss}:
\begin{equation}\label{eq:e2e-loss}
    \mathcal{L}_{\mathrm{e2e}}(f) = \mathbb{E}_{(X, Z) \sim \mathcal{D}} \left[ \mathbf{1}\bigl[ f^{(\Tout)}(X)[-1] \neq Z[-1] \bigr] \right],
\end{equation}
where $f^{(\Tout)}(X)[-1]$ denotes the final token after $\Tout$ steps of autoregressive generation.

Thus training checks predictions on ground-truth prefixes, while end-to-end evaluation rolls out the model on its own prefixes and keeps only the final token.

\section{VC-Dimension of Transformers}\label{sec:vc-results}

For VC-dimension analysis, we focus on the binary classifier \fixed{family} $\Ffam^{01}_{L,W}$ with context length $T$, depth $L$, and $W$ parameters. \fixed{The upper bound is uniform over every fixed class in this family, while the lower bound is witnessed by one fixed class.} The extension to multi-class settings is standard.

\subsection{Upper Bound}\label{subsec:upper-bound}

We first establish an upper bound. The main idea is that hard attention acts as a \emph{discrete switch}: once the attention rankings and ReLU activation patterns are fixed, the network output is polynomial in the parameters. We can then extend the piecewise-polynomial framework for ReLU networks~\citep{bartlett1998almost} to account for these additional switches.

\begin{theorem}[VC Dimension Upper Bound]
\label{thm:upperbound}
For any depth $L \ge 1$, parameter budget \fixed{$W \ge L$}, and context length $T \ge 1$, let $\Ffam^{01}_{L,W}$ be the \fixed{binary Transformer family} defined in \Cref{subsec:Transformers}. \fixed{Then, for every $\mathcal{F}^{01}\in\Ffam^{01}_{L,W}$,}
\begin{equation}
\mathrm{VCdim}(\mathcal{F}^{01}) = \mathcal{O}(W L \log(TW)).
\end{equation}
\end{theorem}

\paragraph{Proof Sketch.}
We proceed in two steps: first partition the parameter space by discrete branching patterns, then bound the number of output sign patterns inside each region.

\paragraph{Step 1: Piecewise-Polynomial Framework.}
We bound the \emph{growth function} $\Pi_{\mathcal{F}^{01}}(H_{1:N})$, which counts the number of distinct output patterns achievable over $N$ embedding sequences; if these $N$ sequences are shattered, this number is $2^N$. We partition the parameter space $\mathbb{R}^W$ into regions where all discrete branching decisions are fixed. Within each region, the network output is polynomial in $\theta$, so $\Pi_{\mathcal{F}^{01}}(H_{1:N}) \le |\mathcal{P}|\cdot\max_C\Pi_C$, where $|\mathcal{P}|$ is the number of regions and $\Pi_C$ is the number of output patterns within region $C$. Warren's theorem bounds $\Pi_C$: for $m$ polynomials of degree $D$ in $W$ variables, the number of distinct sign vectors is at most $(\mathcal{O}(mD/W))^W$.

\paragraph{Step 2: Counting Branching Polynomials.}
At each layer, we identify the \emph{branching polynomials} whose signs determine all discrete choices. For the $\mathsf{FFN}$, these are \fixed{all internal} ReLU pre-activations ($\mathcal{O}(NTW)$ polynomials). For attention, these are the pairwise score differences $g_{n,i,p,q}^{(\ell)}(\theta) := \langle q_{n,i}(\theta), k_{n,p}(\theta) - k_{n,q}(\theta) \rangle$, whose signs determine the relative ranking of positions $p$ and $q$. \fixed{A $W$-parameter architecture has at most $\mathcal{O}(W)$ nontrivial heads;} with $T$ positions, each head and each query requires $\binom{T}{2} = \mathcal{O}(T^2)$ comparisons, yielding $\mathcal{O}(NWT^3)$ additional branching polynomials per layer. Applying Warren's bound and taking logarithms, the $T^3$ term contributes $WL \log T$ to the exponent, yielding the stated bound. The complete derivation appears in \Cref{app:upper-bound-proof}. 

The appendix proves a slightly sharper fixed-architecture statement with logarithmic factor $\log(TW(\mathcal A)L(\mathcal A))$; using $L(\mathcal A)\le L\le W$ and $W(\mathcal A)\le W$ gives the displayed form. \hfill$\blacksquare$

\subsection{Lower Bound}\label{subsec:lower-bound}

We now construct a Transformer that simulates a Recursive Retrieval Machine (\Cref{fig:rrm}). The context stores all $2^B$ possible $B$-bit label configurations, where $B = \lfloor \log_2 (T-1) \rfloor$. For each group-layer pair, the parameters specify which configuration to retrieve. One attention lookup then selects one of $T$ positions, realizing $B=\Theta(\log T)$ labels across the corresponding bit-indexed samples. Chaining this lookup across $L$ layers and $n=\Theta(W)$ groups gives the $\Omega(WL\log T)$ term. Our bound improves the $\Omega(\log T)$ bound from \citet{edelman2022inductive}.

\begin{theorem}[VC Dimension Lower Bound]
\label{thm:lowerbound}
\fixed{There exists an absolute constant $C_0>0$ such that} for any depth $L \ge 1$, context length \fixed{$T \ge 3$}, and parameter budget \fixed{$W \ge C_0 L$}, let $\Ffam^{01}_{L,W}$ be the \fixed{binary Transformer family} defined in \Cref{subsec:Transformers}. \fixed{Then there exists $\mathcal{F}^{01}\in\Ffam^{01}_{L,W}$ such that}
\begin{equation}
\mathrm{VCdim}(\mathcal{F}^{01}) \ge c \cdot W L \left( \log T + \log(W/L) \right),
\end{equation}
where $c > 0$ is an absolute constant.
\end{theorem}

\paragraph{Proof Sketch.}
We construct a set of $N = nLB$ samples that can be shattered, where $n = \Theta(W)$ is the number of sample groups, $L$ is the depth, and $B = \lfloor \log_2(T-1) \rfloor$. Each sample is indexed by a triple $(j, \ell, i)$ with $j \in [n]$, $\ell \in [L]$, $i \in [B]$. The construction below gives the $WL\log T$ term; the $WL\log(W/L)$ term comes from the standard ReLU-network lower bound, since Transformers contain ordinary position-wise ReLU networks as a subclass. We describe the retrieval construction in three steps: (1) how the context encodes labels, (2) how parameters specify retrieval targets, and (3) how attention and FFN perform recursive retrieval.

\begin{figure}[t]
\centering
\resizebox{0.94\textwidth}{!}{%
\begin{tikzpicture}[
  >=Latex,
  font=\small,
  tok/.style={circle,draw,minimum size=5mm,inner sep=0pt},
  qtok/.style={circle,draw,very thick,fill=gray!15,minimum size=5.5mm,inner sep=0pt},
  sel/.style={circle,draw,very thick,fill=gray!15,minimum size=5mm,inner sep=0pt},
  arr/.style={->,thick},
  flow/.style={->,thick,gray!55},
  darr/.style={->,thick,dashed},
  dotarr/.style={densely dotted,thick,gray!55},
  lab/.style={font=\footnotesize,align=left}
]

\def\xs{1.6}  
\def\xd{0.7}  
\def\ys{-0.95} 

\def\xA{0}           
\def\xB{1.6}         
\def\xC{2.7}         
\def\xD{3.8}         
\def\xE{4.9}         
\def\xF{6.0}         
\def\xG{7.1}         
\def\xH{8.2}         
\def\xI{9.8}         

\node[tok] (n00) at (\xA, 0*\ys) {};
\node[tok] (n01) at (\xB, 0*\ys) {};
\node (n02) at (\xC, 0*\ys) {$\cdots$};
\node[sel] (n03) at (\xD, 0*\ys) {};
\node (n04) at (\xE, 0*\ys) {$\cdots$};
\node[tok] (n05) at (\xF, 0*\ys) {};
\node (n06) at (\xG, 0*\ys) {$\cdots$};
\node[tok] (n07) at (\xH, 0*\ys) {};
\node[qtok] (n08) at (\xI, 0*\ys) {};

\node[tok] (n10) at (\xA, 1*\ys) {};
\node[tok] (n11) at (\xB, 1*\ys) {};
\node (n12) at (\xC, 1*\ys) {$\cdots$};
\node[tok] (n13) at (\xD, 1*\ys) {};
\node (n14) at (\xE, 1*\ys) {$\cdots$};
\node[tok] (n15) at (\xF, 1*\ys) {};
\node (n16) at (\xG, 1*\ys) {$\cdots$};
\node[tok] (n17) at (\xH, 1*\ys) {};
\node[qtok] (n18) at (\xI, 1*\ys) {};

\node[tok] (n20) at (\xA, 2*\ys) {};
\node[tok] (n21) at (\xB, 2*\ys) {};
\node (n22) at (\xC, 2*\ys) {$\cdots$};
\node[tok] (n23) at (\xD, 2*\ys) {};
\node (n24) at (\xE, 2*\ys) {$\cdots$};
\node[tok] (n25) at (\xF, 2*\ys) {};
\node (n26) at (\xG, 2*\ys) {$\cdots$};
\node[tok] (n27) at (\xH, 2*\ys) {};
\node[qtok] (n28) at (\xI, 2*\ys) {};

\node (n30) at (\xA, 3*\ys) {$\vdots$};
\node (n31) at (\xB, 3*\ys) {$\vdots$};
\node (n32) at (\xC, 3*\ys) {};
\node (n33) at (\xD, 3*\ys) {$\vdots$};
\node (n34) at (\xE, 3*\ys) {};
\node (n35) at (\xF, 3*\ys) {$\vdots$};
\node (n36) at (\xG, 3*\ys) {};
\node (n37) at (\xH, 3*\ys) {$\vdots$};
\node (n38) at (\xI, 3*\ys) {$\vdots$};

\node[tok] (n40) at (\xA, 4*\ys) {};
\node[tok] (n41) at (\xB, 4*\ys) {};
\node (n42) at (\xC, 4*\ys) {$\cdots$};
\node[tok] (n43) at (\xD, 4*\ys) {};
\node (n44) at (\xE, 4*\ys) {$\cdots$};
\node[sel] (n45) at (\xF, 4*\ys) {};
\node (n46) at (\xG, 4*\ys) {$\cdots$};
\node[tok] (n47) at (\xH, 4*\ys) {};
\node[qtok] (n48) at (\xI, 4*\ys) {};

\node[tok] (n50) at (\xA, 5*\ys) {};
\node[tok] (n51) at (\xB, 5*\ys) {};
\node (n52) at (\xC, 5*\ys) {$\cdots$};
\node[tok] (n53) at (\xD, 5*\ys) {};
\node (n54) at (\xE, 5*\ys) {$\cdots$};
\node[tok] (n55) at (\xF, 5*\ys) {};
\node (n56) at (\xG, 5*\ys) {$\cdots$};
\node[tok] (n57) at (\xH, 5*\ys) {};
\node[qtok] (n58) at (\xI, 5*\ys) {};

\node[tok] (n60) at (\xA, 6*\ys) {};
\node[tok] (n61) at (\xB, 6*\ys) {};
\node (n62) at (\xC, 6*\ys) {$\cdots$};
\node[tok] (n63) at (\xD, 6*\ys) {};
\node (n64) at (\xE, 6*\ys) {$\cdots$};
\node[tok] (n65) at (\xF, 6*\ys) {};
\node (n66) at (\xG, 6*\ys) {$\cdots$};
\node[tok] (n67) at (\xH, 6*\ys) {};
\node[qtok] (n68) at (\xI, 6*\ys) {};

\node[lab] at (\xE, 1.4) {\small\textbf{Context}};
\node[lab] at (\xI, 1.4) {\small\textbf{Query}};
\node[lab, above=6mm of n00, anchor=base] {\small$t \,{=}\, 0\vphantom{t^*_1}$};
\node[lab, above=6mm of n01, anchor=base] {\small$t \,{=}\, 1\vphantom{t^*_1}$};
\node[lab, above=6mm of n02, anchor=base] {\small$\cdots\vphantom{t^*_1}$};
\node[lab, above=6mm of n03, anchor=base] {\small$t \,{=}\, t^*_1$};
\node[lab, above=6mm of n04, anchor=base] {\small$\cdots\vphantom{t^*_1}$};
\node[lab, above=6mm of n05, anchor=base] {\small$t \,{=}\, t^*_L$};
\node[lab, above=6mm of n06, anchor=base] {\small$\cdots\vphantom{t^*_1}$};
\node[lab, above=6mm of n07, anchor=base] {\small$t \,{=}\, T{-}2\vphantom{t^*_1}$};
\node[lab, above=5.5mm of n08, anchor=base] {\small$t \,{=}\, T{-}1\vphantom{t^*_1}$};

\node[lab, left=5mm of n00, anchor=east] {\small$H^{(0)}$};
\node[lab, left=5mm of n10, anchor=east] {\small$\tilde H^{(1)}$};
\node[lab, left=5mm of n20, anchor=east] {\small$H^{(1)}$};
\node[lab, left=5mm of n40, anchor=east] {\small$H^{(L-1)}$};
\node[lab, left=5mm of n50, anchor=east] {\small$\tilde H^{(L)}$};
\node[lab, left=5mm of n60, anchor=east] {\small$H^{(L)}$};
\node[lab, left=35mm of n00, anchor=west] {\small$\mathsf{FFN}^{(0)}$:};
\node[lab, left=35mm of n10, anchor=west] {\small$\mathsf{Attn}^{(1)}$:};
\node[lab, left=35mm of n20, anchor=west] {\small$\mathsf{FFN}^{(1)}$:};
\node[lab, left=35mm of n40, anchor=west] {\small$\mathsf{FFN}^{(L-1)}$:};
\node[lab, left=35mm of n50, anchor=west] {\small$\mathsf{Attn}^{(L)}$:};
\node[lab, left=35mm of n60, anchor=west] {\small$\mathsf{FFN}^{(L)}$:};

\draw[flow] (n00) -- (n10); \draw[flow] (n10) -- (n20); \draw[flow] (n40) -- (n50); \draw[flow] (n50) -- (n60);
\draw[flow] (n01) -- (n11); \draw[flow] (n11) -- (n21); \draw[flow] (n41) -- (n51); \draw[flow] (n51) -- (n61);
\draw[flow] (n03) -- (n13); \draw[flow] (n13) -- (n23); \draw[flow] (n43) -- (n53); \draw[flow] (n53) -- (n63);
\draw[flow] (n05) -- (n15); \draw[flow] (n15) -- (n25); \draw[flow] (n45) -- (n55); \draw[flow] (n55) -- (n65);
\draw[flow] (n07) -- (n17); \draw[flow] (n17) -- (n27); \draw[flow] (n47) -- (n57); \draw[flow] (n57) -- (n67);
\draw[arr] (n08) -- (n18); \draw[arr] (n18) -- (n28); \draw[arr] (n48) -- (n58); \draw[arr] (n58) -- (n68);

\node[draw, rounded corners, inner sep=4pt, font=\small, right=4mm of n08] (emb) {\textbf{Embed}: $r_0 \leftarrow w_j$};
\draw[arr] (emb) -- (n08);

\draw[darr] (n08) -- node[lab, above, pos=0.55, yshift=3pt] {\small\textbf{Match}: $t^*_1=\arg\min_t (t-r_0)^2$} (n03);
\draw[arr]  (n03) -- node[lab, below, sloped, pos=0.5] {\small\textbf{Read}: $b_1=\mathrm{bit}_i(t^*_1)$} (n18);
\node[lab, right=4mm of n18] {\small\textbf{Retrieval}};
\node[lab, right=4mm of n28, align=left] {\small\textbf{Shift}:\\$r_1{=}2T(r_0{-}t^*_1)$};

\node[lab, right=4mm of n48, align=left] {\small\textbf{Shift}:\\$r_{L-1}{=}2T(r_{L-2}{-}t^*_{L-1})$};
\draw[darr] (n48) -- node[lab, above, pos=0.55, yshift=3pt] {\small\textbf{Match}: $t^*_L=\arg\min_t (t-r_{L-1})^2$} (n45);
\draw[arr]  (n45) -- node[lab, below, sloped, pos=0.5] {\small\textbf{Read}: $b_L=\mathrm{bit}_i(t^*_L)$} (n58);
\node[lab, right=4mm of n58] {\small\textbf{Retrieval}};

\node[draw, rounded corners, inner sep=4pt, font=\small, right=4mm of n68] (out) {\textbf{Output}: $\hat y = b_\ell$};
\draw[arr] (n68) -- (out);

\end{tikzpicture}%
}
\caption{Recursive retrieval gadget for sample $(j,\ell,i)$. The context positions enumerate all $B$-bit label configurations, while the query register $r_0$ stores an encoded list of target addresses $(s_{j,1},\ldots,s_{j,L})$. At layer $k$, hard attention selects the context position $t_k^*$ nearest to the current register, the value head reads its $i$-th bit, and the FFN shifts the register to expose the next address. A gate outputs the bit read at the target layer $\ell$.}
\label{fig:rrm}
\end{figure}
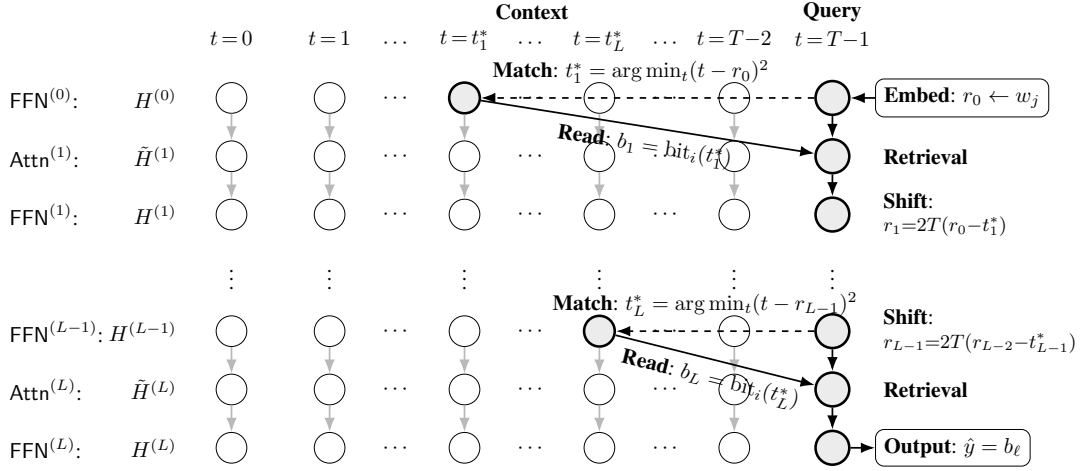

\paragraph{Step 1: Context as Label Repository.}
For each fixed pair $(j,\ell)$, the $B=\lfloor \log_2(T-1)\rfloor$ labels indexed by $i\in[B]$ are encoded by one retrieved context position. Position $t \in \{0,\ldots,T-2\}$ stores the $t$-th labeling by encoding $\mathrm{bit}_i(t)$, so the context enumerates all $2^B\le T-1$ label configurations. For target labels $Y_{j,\ell,1},\ldots,Y_{j,\ell,B}$, the matching position is $t^*=\sum_i Y_{j,\ell,i}2^{i-1}$.

\paragraph{Step 2: Parameters as Target Addresses.}
For each group-layer pair, the parameters store the lookup position. The labels $\{Y_{j,\ell,i}\}_{i=1}^B$ are aggregated into $s_{j,\ell}=\sum_k Y_{j,\ell,k}2^{k-1}$. A trainable affine compression maps each group identifier $j$ to an encoded pointer containing all $L$ addresses $(s_{j,1}, \ldots, s_{j,L})$.

\paragraph{Step 3: Recursive Retrieval via Attention + FFN.}
Given a register $r$ encoding target address $s$, attention selects the context position $t^*=s$: keys extract $(2t,-t^2)$, the query extracts $(r,1)$, and the score $2rt-t^2$ is maximized at the integer nearest to $r$. The value projection reads $\mathrm{bit}_i(t^*)$. The base-$2T$ encoding keeps $|r-s|<1/2$, so $s$ is the unique nearest integer.

To chain $L$ retrievals, set $r_0:=\sum_{m=1}^L s_m/(2T)^{m-1}$. After retrieval step $k$, the $\mathsf{FFN}$ performs the left-shift $r_k=2T(r_{k-1}-t^*)$, exposing the next address. A gate records the retrieved bit only when $k=\ell$.

\paragraph{Putting It Together.}
This affine compression uses $\Theta(n)$ parameters to store $n$ address sequences; the backbone requires only $\mathcal{O}(L)$ parameters. For a sufficiently large absolute constant $C_0$, the regime $W \ge C_0L$ lets us choose $n = \Theta(W)$. We therefore shatter $nLB$ samples, yielding $\Omega(WL\log T)$. Together with the ReLU-subclass lower bound $\Omega(WL\log(W/L))$, this proves \Cref{thm:lowerbound}. The complete construction appears in \Cref{app:lower-bound-proof}. \hfill$\blacksquare$

\section{Auto-Regressive CoT with Transformers}
\label{sec:cot}

 We now turn to the \emph{autoregressive} setting, where next-token generators $f: \Sigma^* \to \Sigma$ iteratively produce variable-length token sequences. 
 This raises a natural question: \emph{how does the availability of full CoT traces during training affect learnability?} 
We first formalize CoT learnability, then introduce the two dimensions governing the generic upper and lower bounds, and finally instantiate both for Transformers.

\subsection{Chain-of-Thought Learnability} \label{subsec:cot-setup}

We begin by formalizing the CoT learning problem. Fix an input domain $\mathcal{X} \subseteq \Sigma^{\Tin}$. 
The base class $\mathcal{F}$ of next-token generators induces the trace and end-to-end hypothesis classes:
\begin{equation}
    \mathcal{F}^{(\Tout)} = \left\{ X \mapsto f^{(\Tout)}(X)[-\Tout{:}] : f \in \mathcal{F} \right\}, \quad
    \mathcal{F}^{(\Tout)}_{\sete} = \left\{ X \mapsto f^{(\Tout)}(X)[-1] : f \in \mathcal{F} \right\}.
\end{equation}

Our goal is to learn the end-to-end class $\mathcal{F}^{(\Tout)}_{\sete}$. We consider the \emph{realizable} setting: there exists a ground-truth generator $f_* \in \mathcal{F}$. Given an input distribution $\mathcal{D}$ over $\mathcal{X}$, we label each input $X \sim \mathcal{D}$ with the full CoT trace $Z = {f_*}^{(\Tout)}(X)[-\Tout{:}]$, and seek to output \fixed{a final-token predictor} with small end-to-end error. Realizability means $f_*$ predicts every token in the trace correctly: $f_*(X \concat Z[{:}t]) = Z[t]$ for all $t \in \{0,\ldots,\Tout-1\}$. Formally, adopting the definition from~\citet{joshi2025cot}:

\begin{definition}[Realizable Chain-of-Thought Learnability]
\label{def:learnable-with-CoT}
\fixed{We say $\mathcal{F}^{(\Tout)}_{\sete}$ is \emph{$\sCoT$-learnable} with sample complexity $m(\varepsilon,\delta)$ if there exists a learning rule $A:(\mathcal X\times\Sigma^{\Tout})^*\to\Sigma^{\mathcal X}$ such that for every distribution $\mathcal D$ over $\mathcal X$ and $f_*\in\mathcal F$, given $m\ge m(\varepsilon,\delta)$ samples $S=\{(X_i,Z_i)\}_{i=1}^m$ with $X_i\sim\mathcal D$ and $Z_i={f_*}^{(\Tout)}(X_i)[-\Tout{:}]$, with probability at least $1-\delta$,}
\begin{equation}
\fixed{
\Pr_{X\sim\mathcal D}\!\left[A(S)(X)\neq {f_*}^{(\Tout)}(X)[-1]\right]\le \varepsilon .
}
\end{equation}
\end{definition}

\paragraph{Learning Algorithm.}
When full CoT traces are available, a natural approach is to train the model to match them step by step. This leads to teacher forced training (\Cref{subsec:autoregressive}) as introduced in \Cref{sec:background}. In the realizable setting, it reduces to finding any generator consistent with all observed CoT traces. Formally, we define the learner as
\begin{equation}
\begin{aligned}
\mathrm{Cons}_{\sCoT}(S):\quad
&\text{Return } \hat f\in\mathcal F \text{ such that}\\
&\hat f(X_i\concat Z_i[{:}t])=Z_i[t]
\quad\text{for all }(X_i,Z_i)\in S,\ t\in\{0,\ldots,\Tout-1\}.
\end{aligned}
\end{equation}
\fixed{This is a proper learner: it returns a generator $\hat f\in\mathcal F$, whose induced final-token predictor is $X\mapsto \hat f^{(\Tout)}(X)[-1]$.}
In standard PAC learning, training and test objectives coincide, so ERM is universal, achieving optimal distribution-free sample complexity. For CoT learning, however, there is a fundamental asymmetry: the learner is trained on full traces but evaluated only on the final token. The intermediate tokens carry information that could guide learning, but a learning algorithm might exploit them without matching every step, and at test time, a different reasoning path could still lead to the correct final answer. This raises a natural question: is teacher forcing the universal approach for CoT learning? 

\fixed{A direct generic bound would lose an extra logarithmic factor in $\Tout$. The next two subsections remove this loss for Transformers and show that the resulting rate is optimal.}

\subsection{Dual Complexity Measures} \label{subsec:complexity-measures}

To characterize CoT sample complexity, we introduce two measures that capture the two sides of the train-test asymmetry: the \emph{trace shattering dimension} $\TSdim$ for the upper bound, and the \emph{answer shattering dimension} $\ASdim$ for the CoT lower bound. 

The first measure captures the complexity of predicting the \emph{entire trace} correctly, which requires that every token along the reasoning path match the target. Equivalently, it is the VC dimension of the binary loss class that asks whether a generated trace matches a reference trace.

\begin{definition}[Trace Shattering Dimension]
\label{def:tsdim}
A set $S = \{X_1, \ldots, X_n\} \subseteq \mathcal{X}$ is \emph{trace shattered} by $\mathcal{F}$ if there exists a trace assignment $R_S: S \to \Sigma^{\Tout}$ such that for every $b \in \{0,1\}^n$, some $f_b \in \mathcal{F}$ satisfies for all $i \in [n]$:
\begin{equation}
b_i = 1 \Longrightarrow f_b^{(\Tout)}(X_i)[-\Tout{:}] = R_S(X_i), \qquad
b_i = 0 \Longrightarrow f_b^{(\Tout)}(X_i)[-\Tout{:}] \neq R_S(X_i).
\end{equation}
The \emph{trace shattering dimension} $\TSdim(\mathcal{F}; \mathcal{X}, \Tout)$ is the size of the largest such set.
\end{definition}

The second measure captures the complexity of predicting only the \emph{final answer}, namely the last token of the trace. The key constraint is that all generators share the same trace prefix, so the intermediate tokens provide no information about the final answer.

\begin{definition}[Answer Shattering Dimension]
\label{def:tidim}
\fixed{Assume $0,1\in\Sigma$. A set $S=\{X_1,\ldots,X_n\}\subseteq\mathcal X$ is \emph{answer shattered} if there exists a prefix assignment $R_S:S\to\Sigma^{\Tout-1}$ such that for every $b\in\{0,1\}^n$, some $f_b\in\mathcal F$ makes the first $\Tout-1$ generated tokens equal to this prefix and sets the final token to $b_i$:}
\begin{equation}
{f_b}^{(\Tout)}(X_i)[-\Tout{:}-1] = R_S(X_i), \qquad {f_b}^{(\Tout)}(X_i)[-1] = b_i.
\end{equation}
The \emph{answer shattering dimension} $\ASdim(\mathcal{F}; \mathcal{X}, \Tout)$ is the size of the largest such set.
\end{definition}

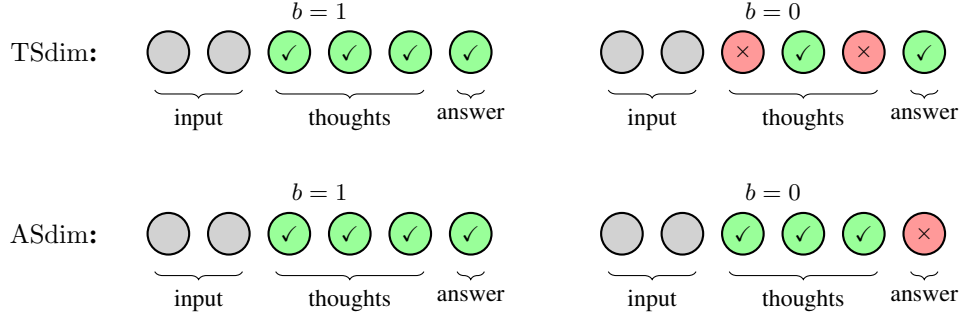
\begin{figure}[t]
\centering
\begin{tikzpicture}[
  >=Latex,
  font=\small,
  tok/.style={circle,draw,thick,minimum size=5.5mm,inner sep=0pt},
  gray/.style={circle,draw,thick,fill=gray!35,minimum size=5.5mm,inner sep=0pt},
  green/.style={circle,draw,thick,fill=green!40,minimum size=5.5mm,inner sep=0pt,font=\scriptsize\bfseries},
  red/.style={circle,draw,thick,fill=red!40,minimum size=5.5mm,inner sep=0pt,font=\scriptsize\bfseries},
  lab/.style={font=\footnotesize}
]

\def\sp{0.8}  
\def\mark{\checkmark}
\def\cross{$\times$}

\node[font=\small\bfseries,anchor=east] at (-0.8, 0) {$\TSdim$:};

\node[lab] at (2.5*\sp, 0.55) {$b = 1$};
\node[gray] (t1_0) at (0*\sp, 0) {};
\node[gray] (t1_1) at (1*\sp, 0) {};
\node[green] (t1_2) at (2*\sp, 0) {\mark};
\node[green] (t1_3) at (3*\sp, 0) {\mark};
\node[green] (t1_4) at (4*\sp, 0) {\mark};
\node[green] (t1_5) at (5*\sp, 0) {\mark};

\draw[decorate,decoration={brace,amplitude=3pt,mirror}] (-0.18, -0.5) -- (1*\sp+0.18, -0.5) node[midway,below=3pt,font=\footnotesize] {input};
\draw[decorate,decoration={brace,amplitude=3pt,mirror}] (2*\sp-0.18, -0.5) -- (4*\sp+0.18, -0.5) node[midway,below=3pt,font=\footnotesize] {thoughts};
\draw[decorate,decoration={brace,amplitude=3pt,mirror}] (5*\sp-0.18, -0.5) -- (5*\sp+0.18, -0.5) node[midway,below=3pt,font=\footnotesize] {answer};

\def\xoff{6.0}
\node[lab] at (\xoff+2.5*\sp, 0.55) {$b = 0$};
\node[gray] (t0_0) at (\xoff+0*\sp, 0) {};
\node[gray] (t0_1) at (\xoff+1*\sp, 0) {};
\node[red] (t0_2) at (\xoff+2*\sp, 0) {\cross};
\node[green] (t0_3) at (\xoff+3*\sp, 0) {\mark};
\node[red] (t0_4) at (\xoff+4*\sp, 0) {\cross};
\node[green] (t0_5) at (\xoff+5*\sp, 0) {\mark};

\draw[decorate,decoration={brace,amplitude=3pt,mirror}] (\xoff-0.18, -0.5) -- (\xoff+1*\sp+0.18, -0.5) node[midway,below=3pt,font=\footnotesize] {input};
\draw[decorate,decoration={brace,amplitude=3pt,mirror}] (\xoff+2*\sp-0.18, -0.5) -- (\xoff+4*\sp+0.18, -0.5) node[midway,below=3pt,font=\footnotesize] {thoughts};
\draw[decorate,decoration={brace,amplitude=3pt,mirror}] (\xoff+5*\sp-0.18, -0.5) -- (\xoff+5*\sp+0.18, -0.5) node[midway,below=3pt,font=\footnotesize] {answer};

\def\yoff{-2.4}
\node[font=\small\bfseries,anchor=east] at (-0.8, \yoff) {$\ASdim$:};

\node[lab] at (2.5*\sp, \yoff+0.55) {$b = 1$};
\node[gray] (a1_0) at (0*\sp, \yoff) {};
\node[gray] (a1_1) at (1*\sp, \yoff) {};
\node[green] (a1_2) at (2*\sp, \yoff) {\mark};
\node[green] (a1_3) at (3*\sp, \yoff) {\mark};
\node[green] (a1_4) at (4*\sp, \yoff) {\mark};
\node[green] (a1_5) at (5*\sp, \yoff) {\mark};

\draw[decorate,decoration={brace,amplitude=3pt,mirror}] (-0.18, \yoff-0.5) -- (1*\sp+0.18, \yoff-0.5) node[midway,below=3pt,font=\footnotesize] {input};
\draw[decorate,decoration={brace,amplitude=3pt,mirror}] (2*\sp-0.18, \yoff-0.5) -- (4*\sp+0.18, \yoff-0.5) node[midway,below=3pt,font=\footnotesize] {thoughts};
\draw[decorate,decoration={brace,amplitude=3pt,mirror}] (5*\sp-0.18, \yoff-0.5) -- (5*\sp+0.18, \yoff-0.5) node[midway,below=3pt,font=\footnotesize] {answer};

\node[lab] at (\xoff+2.5*\sp, \yoff+0.55) {$b = 0$};
\node[gray] (a0_0) at (\xoff+0*\sp, \yoff) {};
\node[gray] (a0_1) at (\xoff+1*\sp, \yoff) {};
\node[green] (a0_2) at (\xoff+2*\sp, \yoff) {\mark};
\node[green] (a0_3) at (\xoff+3*\sp, \yoff) {\mark};
\node[green] (a0_4) at (\xoff+4*\sp, \yoff) {\mark};
\node[red] (a0_5) at (\xoff+5*\sp, \yoff) {\cross};

\draw[decorate,decoration={brace,amplitude=3pt,mirror}] (\xoff-0.18, \yoff-0.5) -- (\xoff+1*\sp+0.18, \yoff-0.5) node[midway,below=3pt,font=\footnotesize] {input};
\draw[decorate,decoration={brace,amplitude=3pt,mirror}] (\xoff+2*\sp-0.18, \yoff-0.5) -- (\xoff+4*\sp+0.18, \yoff-0.5) node[midway,below=3pt,font=\footnotesize] {thoughts};
\draw[decorate,decoration={brace,amplitude=3pt,mirror}] (\xoff+5*\sp-0.18, \yoff-0.5) -- (\xoff+5*\sp+0.18, \yoff-0.5) node[midway,below=3pt,font=\footnotesize] {answer};

\end{tikzpicture}
\caption{\textbf{Top} (Trace Shattering Dimension $\TSdim$): for $b=1$ the entire trace matches the ground-truth; otherwise some tokens differ. \textbf{Bottom} (Answer Shattering Dimension $\ASdim$): for each input, all generators share the same thought prefix and differ only in the final answer.}
\label{fig:shattering}
\end{figure}

\paragraph{Main Result.}
Since answer shattering is a stronger requirement (varying the final token while keeping the prefix fixed necessarily varies the full trace), we have $\ASdim \le \TSdim$. Unlike standard PAC learning, where the VC-dimension governs both bounds, here the upper bound scales with $\TSdim$ while the lower bound scales with $\ASdim$; the gap between them determines how tight this generic characterization is.

\begin{theorem}[Generic Sample Complexity Bounds for CoT]
\label{thm:cot_fundamental}
Fix any $(\mathcal{F}, \mathcal{X}, \Tout)$. There exist universal constants $C, c, \delta_0 > 0$ such that:
\begin{enumerate}
\item[(i)] For all $\varepsilon, \delta \in (0,1)$, $\mathcal{F}^{(\Tout)}_{\sete}$ is $\sCoT$-learnable (using the learning rule $\mathrm{Cons}_{\sCoT}$) with sample complexity
\begin{equation}
m(\varepsilon,\delta) \le C \cdot \frac{\TSdim(\mathcal{F}; \mathcal{X}, \Tout) \cdot \log(1/\varepsilon) + \log(1/\delta)}{\varepsilon}.
\end{equation}
\item[(ii)] Let $d_{\mathrm{AS}} := \ASdim(\mathcal{F}; \mathcal{X}, \Tout)$. \fixed{If $d_{\mathrm{AS}} \ge 2$, then} for all $\varepsilon \in (0, 1/8)$ and all $\delta \le \delta_0$, any $\sCoT$-learnability guarantee requires $m(\varepsilon,\delta) \ge c \cdot d_{\mathrm{AS}}/{\varepsilon}$.
\end{enumerate}
\end{theorem}

The proof is deferred to \Cref{app:cot-fundamental-proof}. For the upper bound, we reduce $\TSdim$ to the CoT trace loss class and apply standard VC theory. For the lower bound, we construct a hard distribution on an answer shattered set; since the trace prefix is fixed, each observed trace reveals only its own final token, leaving unseen points unpredictable.

\subsection{Universality of Teacher Forcing for Transformers} \label{subsec:cot-Transformers}

We now instantiate \Cref{thm:cot_fundamental} for Transformer-based next-token generators. The proof follows the same partition-based analysis as \Cref{sec:vc-results}, applied to each intermediate step of the CoT trace. See \Cref{app:cot-upper-bound-proof} for details.

\begin{theorem}[CoT \fixed{Trace-Shattering} Upper Bound]
\label{thm:cot_upper_bound}
\fixed{For any $L\ge1$, $W\ge L$, finite alphabet $|\Sigma|=K$, and $\mathcal{F}\in\Ffam_{L,W}$,} fix $\mathcal{X} \subseteq \Sigma^{\Tin}$, lengths $\Tin,\Tout$, and $\Ttotal := \Tin + \Tout$. \fixed{Then}
\begin{equation}
    \TSdim(\mathcal{F}; \mathcal{X}, \Tout) = O\left( WL \log(\Ttotal WLK) \right).
\end{equation}
\fixed{In particular, this is $O(WL\log(\Ttotal W))$ whenever $K\le \mathrm{poly}(W)$.}
\end{theorem}

\paragraph{A Scratchpad Lower Bound.}
The upper bound shows that $\mathrm{Cons}_{\sCoT}$ \emph{suffices}. But is it \emph{necessary}? Could a smarter algorithm achieve lower sample complexity by exploiting the structure of CoT traces? We show that the answer is \emph{no}: any algorithm requires $\Omega(WL \log \Ttotal / \varepsilon)$ samples for Transformers.

The key idea is that CoT traces can act as a label-invariant scratchpad. The generator writes intermediate tokens that depend on the input but not on the labeling, so observing the trace prefix reveals no information about the final token. Nevertheless, the model can use this scratchpad internally, for example as a lookup table, to shatter $\Omega(\log \Ttotal)$ final answers. We formalize this intuition:

\begin{theorem}[CoT Answer-Shattering Lower Bound via Scratchpad]
\label{thm:cot_shattering}
For any $N \ge 2$, there exists a constant-size alphabet $\Sigma$, an input length $\Tin = \mathcal{O}(\log\log N)$, a generation length $\Tout = \Theta(N \log N)$, and a class $\mathcal{F}$ of constant-depth, constant-width hard-attention Transformers such that, with $\Ttotal := \Tin + \Tout$,
\begin{equation}
\ASdim(\mathcal{F}; \Sigma^{\Tin}, \Tout) = \Omega(\log \Ttotal).
\end{equation}
\end{theorem}

\paragraph{Proof Sketch.}
We shatter $m:=\lfloor\log_2 N\rfloor$ inputs. The generator writes an input-dependent but label-independent scratchpad, then uses one final readout to recover the requested label bit from a universal table. Let $M:=2^m=\Theta(N)$ and $\ell:=\lceil\log_2 m\rceil$. For readability, write the scratchpad using the symbols $\{0,1,\#,\mathrm{END0},\mathrm{END1}\}$, where $\#$ is a row separator and $\mathrm{END0},\mathrm{END1}$ are anchor tokens. The formal proof uses constantly many decorated copies of these tokens to store finite-state carry/borrow information. As in \Cref{thm:lowerbound}, we use positional features containing $\tau$ and $\tau^2$.

Concretely, each \textbf{INPUT} $X_i:=\mathrm{bin}_\ell(i)\#\in\Sigma^{\ell+1}$ encodes the index $i$ in LSB-first binary. The generated scratchpad has two segments:
\begin{itemize}
    \item \textbf{DECODE}: $\mathrm{bin}_\ell(i{-}1) \# \to \cdots \to \mathrm{bin}_\ell(0) \# \to \mathrm{END0}$. Counts down from $i-1$ to $0$. The anchor token $\mathrm{END0}$ is placed at position $p_0(i) = (i+1)(\ell+1)$; the integer $i$ is encoded into $\mathrm{END0}$'s positional encoding.
    \item \textbf{TABLE}: $\mathrm{bin}_m(0) \# \to \mathrm{bin}_m(1) \# \to \cdots \to \mathrm{bin}_m(M{-}1) \# \to \mathrm{END1}$. Enumerates all $M = 2^m$ binary strings in order, serving as a universal lookup table \emph{identical for all inputs and all labelings}.
\end{itemize}
\textbf{Final readout.} The labeling $y\in\{0,1\}^m$ is encoded into parameter $s:=\sum_j y_j2^j$, but $s$ is gated to have no effect until the final step, so the scratchpad is label-invariant. Let $L_{\mathrm{in}}:=\ell+1$ and $L_{\mathrm{out}}:=m+1$ be the row lengths in the input/decode and table segments. At the final step, $\mathsf{Attn}$ retrieves anchor position $p_0(i)$ from $\mathrm{END0}$; since $p_0(i)=(i+1)L_{\mathrm{in}}$, the $\mathsf{FFN}$ recovers $i$ and computes $r=p_0(i)+1+sL_{\mathrm{out}}+i$, the position of the $i$-th bit in the $s$-th table row. A final attention lookup fetches that bit. Since the $s$-th row stores $\mathrm{bin}_m(s)$ and $y=\mathrm{bin}_m(s)$, this bit equals $y_i$.

Since different labelings produce identical traces except for the final token, this achieves answer shattering of $m$ samples. The total generation length is $\Tout = |\mathrm{DECODE}| + |\mathrm{TABLE}| + c = \mathcal{O}(m \cdot \ell) + \Theta(M \cdot m) + c = \Theta(N \log N)$. Thus $m = \Omega(\log \Tout)$. Since $\Tin = \mathcal{O}(\log\log N)$, we have $\Ttotal = \Tin + \Tout = \Theta(\Tout)$ and hence $m = \Omega(\log \Ttotal)$, giving $\ASdim \ge \Omega(\log \Ttotal)$. \hfill$\blacksquare$

\paragraph{Recursive Amplification.} 
Amplifying the scratchpad construction by the recursive retrieval idea of \Cref{thm:lowerbound} gives the following CoT lower bound. The amplification uses $\Theta(W)$ independent prompt groups and $\Theta(L)$ retrieval rounds; each round contributes one scratchpad-style answer-shattering copy.

\begin{corollary}[CoT Lower Bounds for Transformers]
\label{cor:cot-tidim-lower}
\label{cor:cot-sample-lower}
There exist constants $C_0,L_0,c,\delta_0>0$ such that for every $L\ge L_0$, $W\ge C_0L$, and sufficiently large $\Tout$, there exist a finite alphabet $\Sigma$ with $K:=|\Sigma|\le\mathrm{poly}(W)$, an input length $\Tin=O(\log\log\Tout)$, a prompt domain $\mathcal X\subseteq\Sigma^{\Tin}$, and a class $\mathcal F\in\Ffam_{L,W}$ such that, with $\Ttotal=\Tin+\Tout$,
\begin{equation}
\ASdim(\mathcal F;\mathcal X,\Tout)
\ge c\,WL\log\Ttotal.
\end{equation}
Moreover, taking the worst case over token-prompt Transformer classes in $\Ffam_{L,W}$, for all $\varepsilon\in(0,1/8)$, the $\sCoT$-learning sample complexity satisfies $m(\varepsilon,\delta_0)\ge c\,WL\log(\Ttotal W/L)/\varepsilon$. Here ``sufficiently large'' includes $\Tout\ge C WL\log(W/L)$, which is needed only for the token-prompt ReLU witness.
\end{corollary}

\paragraph[Universality of Teacher Forcing]{Universality of Teacher-Forcing / $\mathrm{Cons}_{\sCoT}$.}
\fixed{\Cref{thm:cot_upper_bound,thm:cot_fundamental} show that $\mathrm{Cons}_{\sCoT}$ learns every $\mathcal F\in\Ffam_{L,W}$ with sample complexity on the order of $WL\log(\Ttotal WLK)/\varepsilon$, up to the standard $\log(1/\varepsilon)$ and confidence terms. Conversely, the lower bound has two independent token-prompt sources. The scratchpad construction in \Cref{cor:cot-tidim-lower}, together with \Cref{thm:cot_fundamental}, gives the CoT term $WL\log\Ttotal/\varepsilon$. The remaining $WL\log(W/L)/\varepsilon$ term is an ordinary supervised obstruction, implemented with token prompts: a marker token selects one fixed positional encoding, after which the Transformer simulates a ReLU MLP. Taking the larger of the two witnesses gives the lower bound in \Cref{cor:cot-sample-lower}. Thus teacher forcing is optimal in sample complexity for Transformers.}

In other words, despite the train-test asymmetry, \emph{matching every intermediate token is the right approach}; exploiting the CoT traces in any other way cannot reduce sample complexity.

\section{Discussion}
\label{sec:discussion}

\paragraph{Other Activations.} Our main result~\Cref{thm:upperbound} extends to piecewise-polynomial activations, e.g. gated ReLU ($(x,y)\mapsto x \cdot [y]_+$)~\citep{shazeer2020glu}. The same Warren-based analysis applies, except that the depth dependence becomes $\tilde O(WL^2\log T)$ instead of $\tilde O(WL\log T)$: the final output can have degree exponential in $L$ as a piecewise polynomial of the parameters, and taking its logarithm yields an extra factor of $L$.
\paragraph{Softmax Attention.} Our approach for $O(\log T)$ VC dimension does not extend to softmax attention because of the exponential function. To our knowledge, the best upper bound for VC dimension of softmax Transformers is $\mathsf{poly}(T)$, by applying results such as Theorem 8.14 in \citep{anthony1999neural}.

\paragraph{Open Problems.}
Our results show that teacher forcing is \fixed{optimal in sample complexity} for learning Transformers with CoT supervision. Two questions remain open: (1) Is teacher forcing universally optimal for general hypothesis classes? In other words, can the gap between the trace shattering dimension and the answer shattering dimension be closed? (2) What is the sample complexity when learning the end-to-end mapping without access to intermediate reasoning traces? This setting is discussed by~\citet{joshi2025cot} and may exhibit fundamentally different scaling.

\bibliography{ref}

\appendix

\crefalias{section}{appendix}

\clearpage
\tableofcontents
\clearpage

\clearpage
\input{preliminaries}

\input{tf_ub}

\input{tf_lb}

\input{cot_fundamental_proof}

\input{cot_ub}

\input{cot_lb}

\end{document}

%% file: preliminaries.tex
\section{Technical Preliminaries}\label{app:preliminaries}

This appendix provides the formal definitions of PAC learnability and VC-dimension used throughout the paper.

\subsection{PAC Learning Framework}

We begin with the standard framework of \emph{probably approximately correct} (PAC) learning~\citep{valiant1984theory}.

\paragraph{Binary Classification.}
Let $\mathcal{X}$ be a domain and $\mathcal{F}^{01} \subseteq \{0,1\}^{\mathcal{X}}$ be a hypothesis class of binary classifiers. A learning algorithm $A$ receives $m$ labeled samples $S = \{(x_i, y_i)\}_{i=1}^m$ drawn i.i.d.\ from a distribution $\mathcal{D}$ over $\mathcal{X} \times \{0,1\}$, and outputs a hypothesis $\widehat{f} = A(S)$. The \emph{population risk} of a hypothesis $f$ under distribution $\mathcal{D}$ is
\begin{equation}
\mathcal{L}_{\mathcal{D}}(f) := \mathbb{P}_{(x,y) \sim \mathcal{D}}[f(x) \neq y].
\end{equation}
A distribution $\mathcal{D}$ is \emph{realizable} by $\mathcal{F}^{01}$ if there exists $f_* \in \mathcal{F}^{01}$ such that $\mathcal{L}_{\mathcal{D}}(f_*) = 0$.

\begin{definition}[PAC Learnability]
\label{def:pac-learnable}
A hypothesis class $\mathcal{F}^{01}$ is \emph{PAC learnable} if there exists an algorithm $A$ and a function $m: (0,1) \times (0,1) \to \mathbb{N}$ such that for every $\varepsilon, \delta \in (0,1)$ and every distribution $\mathcal{D}$ over $\mathcal{X} \times \{0,1\}$ that is realizable by $\mathcal{F}^{01}$, given $m(\varepsilon, \delta)$ i.i.d.\ samples $S \sim \mathcal{D}^{m(\varepsilon, \delta)}$, the algorithm outputs $\widehat{f} = A(S)$ satisfying
\begin{equation}
\mathbb{P}_{S \sim \mathcal{D}^{m(\varepsilon, \delta)}}\bigl[\mathcal{L}_{\mathcal{D}}(\widehat{f}) \le \varepsilon\bigr] \ge 1 - \delta,
\end{equation}
where the probability is over the random samples and the algorithm's internal randomness (if any). The function $m(\varepsilon, \delta)$ is called the \emph{sample complexity}.
\end{definition}

\subsection{VC-Dimension}

The VC-dimension characterizes the sample complexity of PAC learning binary classifiers in the realizable setting.

\begin{definition}[Shattering]
\label{def:shattering}
A hypothesis class $\mathcal{F}^{01}$ \emph{shatters} a finite set $S = \{x_1, \ldots, x_n\} \subseteq \mathcal{X}$ if for every labeling $b \in \{0,1\}^n$, there exists $f \in \mathcal{F}^{01}$ such that $f(x_i) = b_i$ for all $i \in [n]$.
\end{definition}

\begin{definition}[VC-Dimension]
\label{def:vc-dimension}
The \emph{VC dimension} of $\mathcal{F}^{01}$, denoted $\mathrm{VCdim}(\mathcal{F}^{01})$, is the size of the largest set that can be shattered by $\mathcal{F}^{01}$. If arbitrarily large sets can be shattered, we set $\mathrm{VCdim}(\mathcal{F}^{01}) = \infty$.
\end{definition}

The fundamental theorem of statistical learning~\citep{vapnik1971uniform,blumer1989learnability} establishes that a class is PAC learnable if and only if it has finite VC-dimension. Moreover, the sample complexity is characterized by the VC-dimension:

\begin{theorem}[Fundamental Theorem of Statistical Learning]
\label{thm:fundamental-vc}
A hypothesis class $\mathcal{F}^{01}$ is PAC learnable over all distributions if and only if $\mathrm{VCdim}(\mathcal{F}^{01}) < \infty$. Furthermore, the sample complexity satisfies:
\begin{equation}
\Omega\left(\frac{\mathrm{VCdim}(\mathcal{F}^{01}) + \log(1/\delta)}{\varepsilon}\right) \le m(\varepsilon, \delta) \le O\left(\frac{\mathrm{VCdim}(\mathcal{F}^{01}) \log(1/\varepsilon) + \log(1/\delta)}{\varepsilon}\right).
\end{equation}
\end{theorem}

%% file: tf_ub.tex
\section{Proof of VC Dimension Upper Bound (\texorpdfstring{\Cref{thm:upperbound}}{Theorem \ref{thm:upperbound}})}\label{app:upper-bound-proof}

We establish the upper bound using the recursive partitioning framework for piecewise-polynomial networks \citep{bartlett1998almost}, extended to handle the data-dependent routing of Average-Hard Attention.

The proof proceeds in two steps. First, in \Cref{app:partition}, we construct a partition of parameter space $\mathbb{R}^W$ into regions where all discrete branching decisions (ReLU activations and attention routing) are fixed. Second, in \Cref{app:counting}, we count the total number of sign patterns by combining Warren's bound on the number of regions with a local polynomial argument.

\subsection{Recursive Parameter Partitioning}
\label{app:partition}

Fix any Hard-Attention Transformer architecture $\mathcal{A}$ (as defined in \Cref{subsec:Transformers}) with depth $L:=L(\mathcal{A})$ and parameter count $W:=W(\mathcal{A})$.
Let $\mathcal{F}$ denote the corresponding binary classifier class
\[
    \mathcal{F} := \mathcal{F}^{01}_{\mathcal{A}} = \{H \mapsto f^{01}_\theta(H) : \theta \in \mathbb{R}^{W}\}.
\]
For any inputs $H_{1:N}:=(H_1,\dots,H_N)$, define
\begin{equation}
\Pi_{\mathcal{F}}(H_{1:N})
~:=~
\big|\{ (f^{01}_\theta(H_1), \dots, f^{01}_\theta(H_N)) : \theta \in \mathbb{R}^W \}\big|.
\end{equation}
The growth function is $\Pi_{\mathcal{F}}(N):=\max_{H_{1:N}} \Pi_{\mathcal{F}}(H_{1:N})$.

Assume $\mathrm{VCdim}(\mathcal{F}) \ge N$. Then there are $N$ embedding sequences $H_1,\dots,H_N \in \mathbb{R}^{T\times d}$ with $\Pi_{\mathcal{F}}(H_{1:N}) = 2^N$.
We will upper bound $\Pi_{\mathcal{F}}(H_{1:N})$ as a function of $N,W,L,T$; requiring $2^N \le \Pi_{\mathcal{F}}(H_{1:N})$ then yields an upper bound on $\mathrm{VCdim}(\mathcal{F})$.

\paragraph{Partition Decomposition.}
Let $\mathcal{P}$ be any partition of $\mathbb{R}^W$ into regions.
For each region $C\in\mathcal{P}$, define
\begin{equation}
\Pi_C := \big|\{ (f^{01}_\theta(H_1), \dots, f^{01}_\theta(H_N)) : \theta \in C \}\big|
\end{equation}
as the number of distinct sign patterns realized by parameters in $C$.
Then trivially
\begin{equation}
\label{eq:total-count-app}
\Pi_{\mathcal{F}}(H_{1:N}) \le \sum_{C \in \mathcal{P}} \Pi_C.
\end{equation}
We construct a partition $\mathcal{P}_L$ (built layer-by-layer) such that: (i) on each $C\in\mathcal{P}_L$, all discrete branching decisions are fixed for $H_{1:N}$, so the pre-threshold output becomes a polynomial in $\theta$; (ii) $|\mathcal{P}_L|$ is controlled by Warren's bound.

\begin{lemma}[Warren's Bound, {\citep[Lemma 2.1]{bartlett1998almost}}]
\label{lem:warren-app}
Let $g_1, \dots, g_m$ be polynomials in $W$ variables, each of degree at most $D$. If $m \ge W$, the number of distinct sign vectors $(\mathrm{sgn}(g_1(\theta)), \dots, \mathrm{sgn}(g_m(\theta))) \in \{-1, 0, 1\}^m$ as $\theta$ varies in $\mathbb{R}^W$ is at most $2(2emD/W)^W$.
\end{lemma}

\paragraph{Layerwise Refinement.}
Let $H^{(\ell)}_{n,t}(\theta)$ denote the layer-$\ell$ hidden representation for input $H_n$ at token position $t$, viewed as a function of the parameters $\theta$.
We define partitions $\mathcal{P}_0,\mathcal{P}_1,\dots,\mathcal{P}_L$ of $\mathbb{R}^W$, where $\mathcal{P}_\ell$ freezes all branching choices in the first $\ell$ layers.
For $\ell\in\{0,\dots,L\}$, we maintain that for every region $C\in\mathcal{P}_\ell$, every sample $n\in[N]$, and every token $t\in\{0,\ldots,T-1\}$, $H^{(\ell)}_{n,t}(\theta)$ coincides on $C$ with a polynomial of degree at most $\deg_\ell$.

\paragraph{Base case.} 
Set $\mathcal{P}_0 := \{\mathbb{R}^W\}$. 
In our VC analysis, the inputs $H_n \in \mathbb{R}^{T \times d}$ are fixed embedding sequences, so $H_{n, t}^{(0)}(\theta)=H_{n,t}$ is constant in $\theta$ (degree $0$).
(The same argument also covers learnable embeddings, in which case $H^{(0)}$ is affine in $\theta$.) 
Thus the inductive property holds with $\deg_0 \le 1$.

\paragraph{Inductive step.} 
Fix a layer $\ell \in [L]$ and consider a parent region $C \in \mathcal{P}_{\ell-1}$. 
By the inductive hypothesis, the inputs to layer $\ell$ (denoted $H^{(\ell-1)}(\theta)$) are fixed polynomials on $C$. 
We refine $C$ into sub-regions by fixing the outcomes of all discrete branching operations at layer $\ell$. 
Since the layer applies attention \emph{before} the $\mathsf{FFN}$ (\Cref{subsec:Transformers}), we do this in two stages: first freeze the discrete routing/tie structure of attention to make $\tilde{H}^{(\ell)} := \mathsf{Attn}^{(\ell)}(H^{(\ell-1)})$ polynomial, then freeze the ReLU activation pattern of the $\mathsf{FFN}^{(\ell)}$ on $\tilde{H}^{(\ell)}$.
If residual connections are present, they only add previous polynomial representations, possibly after an affine projection, to sublayer outputs; this preserves the polynomial invariant and introduces no additional branching polynomials, so the same partition argument applies.

\emph{1. $\mathsf{Attn}$ routing ($\mathcal{B}_{\mathsf{Attn}}$).} 
The Average-Hard Attention mechanism routes information based on the \emph{relative order} of attention scores. To determine the set of maximizers, it suffices to determine the sign of the difference between any pair of scores. 
For each head $h$, sample $n$, query position $i\in\{0,\ldots,T-1\}$, and ordered key pair $0 \le p < q \le i$ (due to causal masking), we define the pairwise difference polynomial:
\begin{equation}
g_{n, i, p, q}^{(\ell, h)}(\theta) := \left\langle q_{n, i}^{(\ell, h)}(\theta),\, k_{n, p}^{(\ell, h)}(\theta) - k_{n, q}^{(\ell, h)}(\theta) \right\rangle.
\end{equation}
Since queries $q$ and keys $k$ are affine transformations of the previous layer's polynomial representations, $g$ is a polynomial on $C$.

We partition $C$ based on the \textbf{ternary signs} ($\sigma \in \{+1, 0, -1\}$) of these difference polynomials. This fine-grained partition is crucial for handling the Average-Hard mechanism:
fixing the ternary signs uniquely determines the strict ordering and equality relations among all allowed keys, which freezes the set of maximizers $M_{n, i} := \{j \le i \mid \text{score}_j = \max_{k \le i} \text{score}_k\}$ for every query.
Once $M_{n, i}$ is fixed, its cardinality becomes a constant integer, and the attention weights $\alpha_{i, j} = \mathbf{1}[j \in M_{n, i}] / |M_{n, i}|$ become fixed rational constants.
Thus, within any refined sub-region, the attention output $\mathsf{Attn}(H)_{n, i} = \sum_{j \in M_{n, i}} \frac{1}{|M_{n, i}|} V_{n, j}^{(\ell, h)}(\theta)$ collapses to a fixed linear combination of the polynomial value vectors. Since the values $V$ are affine transformations of $H^{(\ell-1)}$, this preserves the polynomial property.
The set $\mathcal{B}_{\mathsf{Attn}}$ collects all such pairwise differences. Accounting for causal masking, for each query position $i$ there are $\binom{i+1}{2}$ key pairs, so each head contributes $\sum_{i=0}^{T-1} \binom{i+1}{2} = \mathcal{O}(T^3)$ polynomials. Under a $W$-parameter budget, there are at most $\mathcal{O}(W)$ heads, hence $|\mathcal{B}_{\mathsf{Attn}}| = \mathcal{O}(NWT^3)$.

\emph{2. ReLU activations in the $\mathsf{FFN}$ ($\mathcal{B}_{\mathsf{FFN}}$).} 
Fix any region where the attention routing decisions above are frozen. On such a region, every intermediate representation $\tilde{H}^{(\ell)}_{n,t}(\theta)$ is polynomial in $\theta$ (as argued above).
We refine sequentially through the internal affine/ReLU layers of the layer-$\ell$ $\mathsf{FFN}$. For each sample $n$, token $t$, internal FFN layer $r$, and ReLU hidden unit $u$ in that internal layer, let $a_{n,t,r,u}^{(\ell)}(\theta)$ denote the corresponding pre-activation, computed after the previous internal ReLU masks have been fixed. Since the input to each internal layer is polynomial on the current region, each $a_{n,t,r,u}^{(\ell)}$ is polynomial there. We refine by the signs of all these pre-activations; after all internal masks are fixed, the $\mathsf{FFN}$ is a mask-fixed affine network, so its outputs remain polynomial in $\theta$.
The set $\mathcal{B}_{\mathsf{FFN}}$ collects all such pre-activations across all ReLU units in the $\mathsf{FFN}$. Since $D^{(\ell)}\le D_0=\mathcal{O}(1)$, under a $W$-parameter budget, there are at most $\mathcal{O}(W)$ ReLU units, so $|\mathcal{B}_{\mathsf{FFN}}| = \mathcal{O}(NTW)$.

\emph{Remark (sum vs.\ product).} The quantities $|\mathcal{B}_{\mathsf{Attn}}|$ and $|\mathcal{B}_{\mathsf{FFN}}|$ count polynomials, so they add. The sequential composition ($\mathsf{Attn}$ then $\mathsf{FFN}$) manifests instead in a \emph{product} bound on the number of refined regions, since we refine first by $\mathcal{B}_{\mathsf{Attn}}$ and then by $\mathcal{B}_{\mathsf{FFN}}$.

\paragraph{Refinement Procedure.}
Let $\Delta^{\mathsf{Attn}}_\ell$ be the maximum number of child regions obtained by refining a parent region in $\mathcal{P}_{\ell-1}$ using the ternary sign patterns of $\mathcal{B}_{\mathsf{Attn}}$ (padding with constant polynomials if $|\mathcal{B}_{\mathsf{Attn}}| < W$). By Lemma~\ref{lem:warren-app},
\begin{equation}
\label{eq:Delta-attn-ell-app}
\Delta^{\mathsf{Attn}}_\ell \le 2\Big(\frac{2e|\mathcal{B}_{\mathsf{Attn}}| D_\ell}{W}\Big)^W.
\end{equation}

Next, within any attention-frozen region, let $\Delta^{\mathsf{FFN}}_\ell$ be the maximum number of child regions obtained by refining using the (binary) sign patterns of $\mathcal{B}_{\mathsf{FFN}}$ (again padding with constants if needed). The constant number of sequential internal refinements is absorbed into constants. Applying Lemma~\ref{lem:warren-app} gives
\begin{equation}
\label{eq:Delta-ffn-ell-app}
\Delta^{\mathsf{FFN}}_\ell \le 2\Big(\frac{2e|\mathcal{B}_{\mathsf{FFN}}| D_\ell}{W}\Big)^W.
\end{equation}
Let $\Delta_\ell := \max_{C \in \mathcal{P}_{\ell-1}} |\{C' \in \mathcal{P}_\ell : C' \subseteq C\}|$ denote the maximum number of child regions produced by refining any parent region.
Since we refine sequentially, each parent region produces at most $\Delta_\ell \le \Delta^{\mathsf{Attn}}_\ell \cdot \Delta^{\mathsf{FFN}}_\ell$ children in $\mathcal{P}_\ell$, hence $|\mathcal{P}_L| \le \prod_{\ell=1}^L \Delta_\ell$.

\paragraph{Degree Bounds.}
We have the crude parameter-counting bounds
\[
|\mathcal{B}_{\mathsf{Attn}}| = \mathcal{O}(NWT^3),
\qquad
|\mathcal{B}_{\mathsf{FFN}}| = \mathcal{O}(NTW).
\]

For degrees: within any region where routing decisions and ReLU masks are fixed, a Transformer layer applies an attention sublayer and an $\mathsf{FFN}$ with at most $D_0=\mathcal{O}(1)$ internal affine/ReLU layers, hence only a constant number of affine maps (degree 1 in $\theta$) to polynomial inputs (degree $\deg_{\ell-1}$). Hence $\deg_\ell \le \deg_{\ell-1} + \mathcal{O}(1) = \mathcal{O}(\ell)$ for representations.
For branching polynomials, attention score differences are inner products of queries and keys (each degree $\le \deg_{\ell-1}+\mathcal{O}(1)$), yielding $D_\ell \le 2\deg_{\ell-1} + \mathcal{O}(1) = \mathcal{O}(\ell)$.

\subsection{Final Bound Derivation}
\label{app:counting}

\paragraph{Total Number of Regions.}
Combining \eqref{eq:Delta-attn-ell-app} and \eqref{eq:Delta-ffn-ell-app} across layers:
\begin{equation}
\label{eq:PL-app}
|\mathcal{P}_L|
\le \prod_{\ell=1}^L \Delta_\ell
\le \prod_{\ell=1}^L \Delta^{\mathsf{Attn}}_\ell \cdot \Delta^{\mathsf{FFN}}_\ell
\le \prod_{\ell=1}^L 4\Big(\frac{2e|\mathcal{B}_{\mathsf{Attn}}| D_\ell}{W}\Big)^W \Big(\frac{2e|\mathcal{B}_{\mathsf{FFN}}| D_\ell}{W}\Big)^W.
\end{equation}

\paragraph{Patterns within a Final Region.}
Fix a final region $C \in \mathcal{P}_L$.
By construction, all branching decisions are fixed, so each output logit $z_n(\theta) := \langle w_{\mathrm{out}}, H^{(L)}_{n,T-1}(\theta)\rangle$ is a polynomial in $\theta$.
Since $w_{\mathrm{out}}$ is included in $\theta$, $\deg_{\mathrm{out}} \le \deg_L + 1 = \mathcal{O}(L)$.
If $N<W$, the desired VC upper bound is immediate after increasing constants, so assume $N\ge W$. Applying Lemma~\ref{lem:warren-app} to the $N$ polynomials $\{z_n\}_{n=1}^N$ yields
\begin{equation}
\label{eq:Pi-in-app}
\Pi_C \le 2\Big(\frac{2eN \deg_{\mathrm{out}}}{W}\Big)^W.
\end{equation}

\paragraph{Putting it Together.}
From \eqref{eq:total-count-app}, $\Pi_{\mathcal{F}}(H_{1:N}) \le \sum_{C \in \mathcal{P}_L} \Pi_C \le |\mathcal{P}_L| \cdot \max_{C \in \mathcal{P}_L} \Pi_C$.
Combining with \eqref{eq:PL-app} and \eqref{eq:Pi-in-app}:
\begin{equation}
\Pi_{\mathcal{F}}(H_{1:N})
\le |\mathcal{P}_L|\cdot \max_C \Pi_C
\le \left( \prod_{\ell=1}^L 4\Big(\frac{2e|\mathcal{B}_{\mathsf{Attn}}| D_\ell}{W}\Big)^W \Big(\frac{2e|\mathcal{B}_{\mathsf{FFN}}| D_\ell}{W}\Big)^W \right) \cdot 2\Big(\frac{2eN \deg_{\mathrm{out}}}{W}\Big)^W.
\end{equation}
If $H_{1:N}$ is shattered, then $\Pi_{\mathcal{F}}(H_{1:N})=2^N$. Taking logarithms:
\begin{align}
N
&\le \mathcal{O}(WL)
+ \sum_{\ell=1}^L W \log\Big(\frac{C\,|\mathcal{B}_{\mathsf{Attn}}| D_\ell}{W}\Big)
+ \sum_{\ell=1}^L W \log\Big(\frac{C\,|\mathcal{B}_{\mathsf{FFN}}| D_\ell}{W}\Big)
+ W\log\Big(\frac{C\,N \deg_{\mathrm{out}}}{W}\Big). \label{eq:log-ineq-app}
\end{align}
Substituting $|\mathcal{B}_{\mathsf{Attn}}| \le c_1 NWT^3$, $|\mathcal{B}_{\mathsf{FFN}}| \le c_2 NTW$, and $D_\ell = \mathcal{O}(\ell)$, the factor $W$ cancels with the $/W$:
\begin{equation}
\log\Big(\frac{C\,|\mathcal{B}_{\mathsf{Attn}}| D_\ell}{W}\Big) \le \log(C'\,N T^3 \ell),
\qquad
\log\Big(\frac{C\,|\mathcal{B}_{\mathsf{FFN}}| D_\ell}{W}\Big) \le \log(C''\,N T \ell).
\end{equation}
Also $\deg_{\mathrm{out}}=\mathcal{O}(L)$, so \eqref{eq:log-ineq-app} yields
\begin{equation}
\label{eq:N-upper-app}
N \le \mathcal{O}(WL\log N + WL\log T + WL\log L).
\end{equation}

\paragraph{Inverting the Inequality.}
We use the standard implication: if $A\ge 2$, $B\ge 0$, and $N \le A\log N + B$, then $N = \mathcal{O}(A\log A + B)$.
Applying this to \eqref{eq:N-upper-app} with $A=\mathcal{O}(WL)$ and $B=\mathcal{O}(WL(\log T + \log L))$ gives
\begin{equation}
N \le \mathcal{O}(WL\log(WL) + WL\log T + WL\log L) = \mathcal{O}(WL\log(TWL)).
\end{equation}
Hence $\mathrm{VCdim}(\mathcal{F}) = \mathcal{O}(WL\log(TWL))$.
When $L\le W$, this simplifies to $\mathrm{VCdim}(\mathcal{F}) = \mathcal{O}(WL\log(TW))$.
\hfill \BlackBox

%% file: tf_lb.tex
\section{Proof of VC Dimension Lower Bound (\texorpdfstring{\Cref{thm:lowerbound}}{Theorem \ref{thm:lowerbound}})}\label{app:lower-bound-proof}

We provide a constructive proof. Throughout this section, $W$ denotes the \emph{total} number of trainable parameters.
We work at a fixed sequence length $T$ (the same $T$ in all parts). We index positions by $t\in\{0,\ldots,T-1\}$, where the last token ($t=T-1$) plays the role of the query/output token.
For the Recursive Retrieval Machine construction, we use residual connections, which are allowed in our architecture class.

The proof has two components. First, in \Cref{app:pointer-lb}, we present a Recursive Retrieval Machine construction that yields a class $\mathcal F^{01}_{\mathrm{RRM}}\in\Ffam^{01}_{L,W}$ with VC-dimension $\Omega(WL\log T)$. Second, in \Cref{app:mlp-subclass}, we demonstrate that the same Transformer setup contains a standard deep ReLU-network subclass $\mathcal F^{01}_{\mathrm{ReLU}}\in\Ffam^{01}_{L,W}$ with VC-dimension $\Omega(WL\log(W/L))$. One of these two classes witnesses the final lower bound.

\subsection{Recursive Retrieval Machine Lower Bound \texorpdfstring{$\Omega(WL\log T)$}{Omega(WL log T)}}\label{app:pointer-lb}

\paragraph{Goal and Setup.} Let $B=\lfloor \log_2 (T-1)\rfloor$, so that $2^B \le T-1$. Note that $B=\Theta(\log T)$ for $T\ge 3$, so we freely write $\log T$ in asymptotic bounds. Define a sample set $S$ of size $N := n\cdot L \cdot B$, indexed by $(j,\ell,i)$ with $j\in[n]$, $\ell\in[L]$, $i\in[B]$. 
Fix an arbitrary labeling $Y\in\{0,1\}^{N}$, written as $Y_{j,\ell,i}$. Our goal is to construct weights $\theta$ so that $f^{01}_\theta(H^{(j,\ell,i)})=Y_{j,\ell,i}$ for all points in $S$.

\paragraph{Input Construction.} Since our classifiers take embedding sequences as input, we directly specify inputs $H^{(j,\ell,i)}\in\mathbb{R}^{T\times d^{(0)}}$ of length $T$, consisting of $T$ embedding vectors:
\begin{equation}
H^{(j,\ell,i)} = \big[h_0^{(j,\ell,i)},\, h_1^{(j,\ell,i)},\, \dots,\, h_{T-2}^{(j,\ell,i)},\, h_{T-1}^{(j,\ell,i)}\big]^\top.
\end{equation}
We append two scalar coordinates to every embedding: (i) a constant coordinate $c\equiv 1$ (present in all positions), and (ii) an indicator coordinate $\eta$ (equal to $1$ on context positions and $0$ on the query position).

\noindent\emph{(i) Context sequence.} For each $t\in\{0,\ldots,T-2\}$, define the context embedding
\begin{equation}\label{eq:pointer-context-app}
    h_t^{(j,\ell,i)}
    \;=\;
    \big[t,\; -t^2,\; \mathrm{bit}_i(t),\; c=1,\; \eta=1,\; \mathbf{0}\big]^\top.
\end{equation}
The component $\mathrm{bit}_i(t)\in\{0,1\}$ is the $i$-th bit of integer $t$ (in any fixed convention). Thus position $t \in \{0,\ldots,T-2\}$ stores the bit pattern of integer $t \in \{0,\ldots,T-2\}$.

\noindent\emph{(ii) Query token.} Let $\mathbf{e}_j\in\mathbb{R}^n$ and $\mathbf{e}_\ell\in\mathbb{R}^L$ be one-hot vectors. The query token activates the memory address $j$ and the target layer $\ell$:
\begin{equation}\label{eq:pointer-query-app}
    h_{T-1}^{(j,\ell,i)}
    \;=\;
    \big[0,\;0,\;0,\; c=1,\; \eta=0,\; \mathbf{e}_j^\top,\; \mathbf{e}_\ell^\top,\; \mathbf{0}\big]^\top.
\end{equation}

\paragraph{Parameter Encoding.} For each address $j$ and layer $\ell$, aggregate the $B$ labels into an integer address $s_{j,\ell} := \sum_{k=0}^{B-1} Y_{j,\ell,k+1}\,2^k \in \{0,\dots,2^B-1\} \subseteq \{0,\dots,T-2\}$, which corresponds to context position $s_{j,\ell} \in \{0,\dots,T-2\}$. Then compress $(s_{j,1},\dots,s_{j,L})$ into a single real number via base-$2T$ encoding:
\begin{equation}\label{eq:aj-app}
    a_j
    \;:=\;
    \sum_{m=1}^L \frac{s_{j,m}}{(2T)^m}
    \;\in\; [0,0.5).
\end{equation}
We designate a specific scalar dimension in the hidden state to act as an \emph{address register} $r$. The value $a_j$ is not part of the input; it is stored in the trainable weights. The one-hot vector $\mathbf e_j$ only tells the first block which stored number to use.

\paragraph{Reading the identifiers once.}
The one-hot coordinates $\mathbf e_j$ and $\mathbf e_\ell$ are used only in the first Transformer block. In that block, trainable linear maps read them once: the query map uses $\mathbf e_j$ to produce the scalar address $r_0=2T a_j$ in the first attention score, and the FFN/residual part records the few scalars needed later, such as the updated address register, the target-layer number $g=\ell$, and the accumulator. The context features $(t,-t^2,\mathrm{bit}_i(t),c,\eta)$ already have constant dimension and are carried forward by residual connections. After the first block, the network no longer carries $\mathbf e_j$ or $\mathbf e_\ell$; all later layers operate on a constant number of coordinates. This first read costs $\Theta(n)$ parameters for choosing $a_j$ from $\mathbf e_j$ and $\mathcal O(L)$ parameters for choosing $g$ from $\mathbf e_\ell$.

\paragraph{Layer Dynamics.} We analyze execution through an inductive invariant. We reserve disjoint coordinate subspaces: (1) a \emph{context subspace} that stores (embedded versions of) the context features $t,-t^2,\mathrm{bit}_i(t),c,\eta$ and is never overwritten; (2) a \emph{query work subspace} that stores the scalar register $r_k$, the selected index $t^*$, the extracted bit $b_k$, and an accumulator $z$. The residual path carries the context coordinates unchanged from layer to layer. The attention and FFN updates may read these coordinates, but they write only to the query work coordinates. In particular, the computations that update the query work subspace never overwrite the context features needed for later retrievals.

\begin{lemma}[Inductive Extraction]\label{lem:pointer-inductive-app}
Suppose that at retrieval step $k$ ($1\le k\le L$) the scalar address available to the query projection equals $r_{k-1} = \sum_{m=k}^L \frac{s_{j,m}}{(2T)^{m-k}}$. For $k=1$, this value is produced directly from $\mathbf e_j$ by the first block's query map; for $k\ge2$, it is the register value left by the previous block. Then:
\begin{enumerate}
    \item The attention mechanism uniquely attends to the context token $t^* = s_{j,k}$.
    \item The $\mathsf{FFN}$ updates the register to
    $r_k = \sum_{m=k+1}^L \frac{s_{j,m}}{(2T)^{m-(k+1)}}$.
\end{enumerate}
\end{lemma}

\begin{proof}
\textbf{Attention step (Nearest-Neighbor Retrieval).}
We configure the query/key projections so that the attention score at the query position for a context token $t$ equals
\begin{equation}\label{eq:score-impl-app}
\langle q, k_t\rangle \;=\; 2 r_{k-1}\cdot t \;-\; t^2 \;+\; M\eta_t,
\end{equation}
where $M>0$ is a sufficiently large constant.
Concretely, this is achieved by choosing $k_t = [\,t,\,-t^2,\,\eta_t\,]$ and $q = [\,2 r_{k-1},\,1,\,M\,]$. Observe that maximizing the term $2 r_{k-1} t - t^2$ is equivalent to minimizing the squared Euclidean distance between the register and $t$:
\begin{equation}
\operatorname*{argmax}_{t} (2 r_{k-1} t - t^2) = \operatorname*{argmin}_{t} |t - r_{k-1}|^2.
\end{equation}
The term $M\eta_t$ simply acts as a mask to exclude the query token itself (where $\eta_{T-1}=0$).
Thus, the attention mechanism effectively searches for the position $t \in \{0,\ldots,T-2\}$ such that $t$ is closest to $r_{k-1}$.

Now write $r_{k-1} = s_{j,k} + \delta_k$, where
\[
\delta_k := \sum_{m=k+1}^L \frac{s_{j,m}}{(2T)^{m-k}}.
\]
Using $0\le s_{j,m}\le T-2$, we have
\begin{equation}
0 \le \delta_k \le (T-2)\sum_{p=1}^\infty (2T)^{-p} = \frac{T-2}{2T-1} < 0.5.
\end{equation}
Since $|\delta_k| < 0.5$, the integer $s_{j,k}$ is strictly closer to $r_{k-1}$ than any other integer. Thus $t = s_{j,k}$ is the unique minimizer, i.e., $t^* = s_{j,k}$ is the selected position.
Since the maximizer is unique, Average-Hard Attention reduces to standard deterministic selection.

\paragraph{Update step (Left-Shift Operation).}
We choose the value projection so that the attention output retrieves two quantities: (i) the value $t^*=s_{j,k}$ (read from the {first} coordinate in \eqref{eq:pointer-context-app}), and (ii) the payload bit $b_k=\mathrm{bit}_i(s_{j,k})$ (read from the {third} coordinate).

The $\mathsf{FFN}$ then computes the next register value using $t^* = s_{j,k}$:
\begin{equation}
r_k := 2 T (r_{k-1} - s_{j,k}) = 2 T \delta_k.
\end{equation}
This is the base-$2T$ left shift.
It is affine on the bounded relevant domain, so a two-layer ReLU $\mathsf{FFN}$ implements it exactly.
Substituting the definition of $\delta_k$ confirms that
\[
r_k = \sum_{m=k+1}^L \frac{s_{j,m}}{(2T)^{m-(k+1)}},
\]
satisfying the inductive hypothesis.
\end{proof}

\paragraph{Output Gating via Scalar Addressing.}
To select the output of the target layer $\ell$ without violating the constant-width constraint (i.e., avoiding an $L$-dimensional state), we employ a scalar addressing mechanism. As described above, the first block maps the target index $\mathbf{e}_\ell$ to the scalar $g := \ell$, which is stored in a protected coordinate. At each layer $k$, we implement a localized indicator function using a ReLU ``hat'' construction:
\begin{equation}\label{eq:hat-gate-app}
\operatorname{gate}_k(g):=2\Bigl(\operatorname{ReLU}\left(g-\left(k-\frac{1}{2}\right)\right)-2 \operatorname{ReLU}(g-k)+\operatorname{ReLU}\left(g-\left(k+\frac{1}{2}\right)\right)\Bigr).
\end{equation}
One verifies that for integer $g \in \{1, \dots, L\}$, $\operatorname{gate}_k(g)=\mathbf{1}\{g=k\}$. We then maintain an accumulator $z$ (initialized to 0) via the conditional update:
\begin{equation}
z \leftarrow z+\operatorname{ReLU}\left(b_k+\operatorname{gate}_k(g)-1\right).
\end{equation}
Since $b_k \in \{0, 1\}$, the update term $\operatorname{ReLU}(b_k + \operatorname{gate}_k(g) - 1)$ simplifies to $b_k$ if $k=\ell$ (where $\operatorname{gate}_k(g)=1$) and vanishes otherwise. Thus, the final state satisfies $z = b_{\ell} = Y_{j, \ell, i}$.
Finally, choosing $w_{\mathrm{out}}$ such that the output logit equals $z-\frac{1}{2}$ ensures that the sign activation outputs the correct label.

\paragraph{Parameter Budget Analysis.}
The construction utilizes an asymmetric parameter allocation. The only parameters that scale with $n$ or $L$ are the first-block weights that read $\mathbf e_j$ and $\mathbf e_\ell$: $\Theta(n)$ weights choose the encoded value $a_j$, and $\mathcal O(L)$ weights choose $g=\ell$. In contrast, the Transformer backbone requires only $\mathcal{O}(L)$ parameters, as the $\mathsf{Attn}$ and $\mathsf{FFN}$ weights at each layer implement fixed algebraic operations on a constant-dimensional subspace, independent of $n$ and $T$. Thus, the total parameter count is $W = \Theta(n) + \mathcal{O}(L)$. For a sufficiently large absolute constant $C_0$, the regime $W \ge C_0L$ lets us choose $n = \Theta(W)$. Let $\mathcal F^{01}_{\mathrm{RRM}}\in\Ffam^{01}_{L,W}$ denote the fixed architecture class used in this construction. Since each sample shatters $\Omega(L \log T)$ bits (specifically, $B$ bits across $L$ layers),
\begin{equation}
\mathrm{VCdim}(\mathcal F^{01}_{\mathrm{RRM}}) \ge n \cdot L \cdot B = \Omega(W L \log T).
\end{equation}

\subsection{ReLU-Network Subclass and Combined Lower Bound}\label{app:mlp-subclass}

We complement the Recursive Retrieval Machine construction with a depth-dependent lower bound by identifying a standard ReLU network subclass within the Transformer architecture.

\begin{lemma}[ReLU Subclass Containment]\label{lem:relu-subclass-app}
For any $T \ge 2$, $L \ge 1$, and $W \ge L$, there exists a class $\mathcal F^{01}_{\mathrm{ReLU}}\in\Ffam^{01}_{L,W}$ that contains, as a subclass, standard ReLU networks with $\Theta(L)$ layers and $\Theta(W)$ trainable parameters acting on the query position.
\end{lemma}

\begin{proof}
We restrict to inputs where all context tokens are fixed, and only the query token varies. By choosing attention parameters so that, at the query position, the unique score-maximizing key is the query token itself (e.g., using a marker coordinate that equals $1$ only on the query token), hard attention acts as an identity map on the query token. Hence, on this subclass, the residual-free Transformer reduces to iterated application of the position-wise $\mathsf{FFN}^{(\ell)}$ on the query token, i.e., a deep ReLU MLP (up to constant-factor depth) with $\Theta(W)$ parameters.
\end{proof}

It is a standard result that depth-$L'$ ReLU networks with $W'$ parameters have VC-dimension $\Omega(W'L' \log(W'/L'))$~\citep{bartlett2019nearly}.
Applying this to our subclass with $L' = \Theta(L)$ and $W' = \Theta(W)$ yields a lower bound of $\Omega(WL \log(W/L))$.

\paragraph{Combined Lower Bound.}
Combining the two regimes, one of the two fixed classes witnesses the sum up to constants:
\begin{equation}
\begin{aligned}
&\exists\,\mathcal F^{01}\in\Ffam^{01}_{L,W}\quad\text{s.t.}\\
&
\operatorname{VCdim}(\mathcal F^{01}) \ge \Omega\left( WL \log T + WL \log \frac{W}{L} \right)
= \Omega\left( WL \log \frac{TW}{L} \right).
\end{aligned}
\end{equation}
\hfill \BlackBox

%% file: cot_fundamental_proof.tex
\section{Proof of Generic Sample Complexity Bounds for CoT (\texorpdfstring{\Cref{thm:cot_fundamental}}{Theorem \ref{thm:cot_fundamental}})}\label{app:cot-fundamental-proof}

In this appendix, we provide the detailed proofs for \Cref{thm:cot_fundamental}. We first establish two auxiliary lemmas, then prove the upper and lower bounds.

\paragraph{Full statement (for reference).}
Fix $(\mathcal{F}, \mathcal{X}, \Tout)$ and let $d_{\mathrm{TS}} := \TSdim(\mathcal{F}; \mathcal{X}, \Tout)$ and
$d_{\mathrm{AS}} := \ASdim(\mathcal{F}; \mathcal{X}, \Tout)$.
For every realizable pair $(\mathcal{D}, f_*)$ and every $\varepsilon,\delta\in(0,1)$, if
\begin{equation}
m \ge C \cdot \frac{d_{\mathrm{TS}} \log(1/\varepsilon) + \log(1/\delta)}{\varepsilon},
\end{equation}
then with probability at least $1-\delta$ over an i.i.d.\ sample $S_{\sCoT}\sim\mathcal{D}^m$ labeled by traces
$Z_i = {f_*}^{(\Tout)}(X_i)[-\Tout{:}]$, the output $\hat{f} = \mathrm{Cons}_{\sCoT}(S_{\sCoT})$ satisfies
$\mathcal{L}_{\mathrm{e2e}}(\hat{f}) \le \varepsilon$.
Moreover, if $d_{\mathrm{AS}}\ge 2$, there exist universal constants $c_0,c_1>0$ such that for any learning rule
$A:(\mathcal{X}\times\Sigma^{\Tout})^*\to \Sigma^{\mathcal X}$ and any $\varepsilon\in(0,1/8)$, if
$m < c_0 \cdot d_{\mathrm{AS}}/\varepsilon$, then there exist a realizable distribution $\mathcal{D}$ over $\mathcal{X}$
and a target $f_*\in\mathcal{F}$ such that for $S_{\sCoT}\sim\mathcal{D}^m$ labeled by traces
$Z_i = {f_*}^{(\Tout)}(X_i)[-\Tout{:}]$,
\begin{equation}
\Pr\!\left[\Pr_{X \sim \mathcal{D}}\big[A(S_{\sCoT})(X) \neq {f_*}^{(\Tout)}(X)[-1]\big] \ge \varepsilon\right] \ge c_1.
\end{equation}

\paragraph{CoT loss class.}
For any generator $f \in \mathcal{F}$ and any pair $(X, Z) \in \mathcal{X} \times \Sigma^{\Tout}$, define the binary CoT loss
\begin{equation}
\ell_f(X, Z) := \mathbf{1}\big[f^{(\Tout)}(X)[-\Tout{:}] \neq Z\big].
\end{equation}
Let $\mathcal{L}_{\mathrm{CoT}} := \{\ell_f : f \in \mathcal{F}\}$.

\subsection{Auxiliary Lemmas}

\begin{lemma}[End-to-end error is dominated by CoT loss]\label{lem:e2e_le_cot}
For every $f \in \mathcal{F}$ and every realizable pair $(\mathcal{D}, f_*)$,
\begin{equation}
\mathcal{L}_{\mathrm{e2e}}(f) \le \mathbb{E}_{X\sim \mathcal{D}}\big[\ell_f\big(X, {f_*}^{(\Tout)}(X)[-\Tout{:}]\big)\big].
\end{equation}
\end{lemma}
\begin{proof}
Fix $X \in \mathcal{X}$ and write $Z := {f_*}^{(\Tout)}(X)[-\Tout{:}]\in\Sigma^{\Tout}$.
If $f^{(\Tout)}(X)[-\Tout{:}]=Z$, then in particular $f^{(\Tout)}(X)[-1]=Z[-1]$, so
\begin{equation}
\mathbf{1}\big[f^{(\Tout)}(X)[-1] \neq {f_*}^{(\Tout)}(X)[-1]\big]
\le
\mathbf{1}\big[f^{(\Tout)}(X)[-\Tout{:}] \neq {f_*}^{(\Tout)}(X)[-\Tout{:}]\big]
= \ell_f(X, Z)
\end{equation}
pointwise. Taking expectation over $X\sim\mathcal{D}$ proves the claim.
\end{proof}

\begin{lemma}[$\TSdim$ equals the VC-dimension of the CoT loss class]\label{lem:tcs_equals_vc}
\begin{equation}
\TSdim(\mathcal{F}; \mathcal{X}, \Tout) = \mathrm{VCdim}(\mathcal{L}_{\mathrm{CoT}}),
\end{equation}
where $\mathcal{L}_{\mathrm{CoT}} = \{\ell_f : f \in \mathcal{F}\}$ with $\ell_f(X, Z) = \mathbf{1}[f^{(\Tout)}(X)[-\Tout{:}] \neq Z]$.
\end{lemma}
\begin{proof}
\emph{($\ge$)} Suppose $S = \{X_1, \ldots, X_n\}$ is trace shattered with witness $R_S: S \to \Sigma^{\Tout}$.
Consider the $n$ labeled examples in the loss domain: $\{(X_i, R_S(X_i))\}_{i=1}^n$.
For any $b \in \{0,1\}^n$, by definition there exists $f_b \in \mathcal{F}$ such that
$\ell_{f_b}(X_i, R_S(X_i)) = 0$ when $b_i = 1$ and $= 1$ when $b_i = 0$.
Thus $\mathcal{L}_{\mathrm{CoT}}$ shatters these $n$ points, so $\mathrm{VCdim}(\mathcal{L}_{\mathrm{CoT}}) \ge n$.

\emph{($\le$)} Conversely, suppose $\mathrm{VCdim}(\mathcal{L}_{\mathrm{CoT}}) \ge n$.
Then there exist $(X_i, y^{(i)}_{0:\Tout-1}) \in \mathcal{X} \times \Sigma^{\Tout}$ for $i \in [n]$ such that for every label vector
$a \in \{0,1\}^n$ there exists $f_a \in \mathcal{F}$ with $\ell_{f_a}(X_i, y^{(i)}_{0:\Tout-1}) = a_i$.
We may assume the inputs $X_i$ are distinct. Indeed, if $X_i=X_j$ but $y^{(i)}_{0:\Tout-1}\neq y^{(j)}_{0:\Tout-1}$, no deterministic generator can realize loss $0$ on both points; if the two traces are equal, the two loss-domain points are identical and cannot both belong to a shattered set of distinct points.
Let $S = \{X_1, \ldots, X_n\}$ and define $R_S(X_i) := y^{(i)}_{0:\Tout-1}$.
For any $b \in \{0,1\}^n$, apply shattering with $a_i := 1 - b_i$.
Then $\ell_{f_a}(X_i, R_S(X_i)) = 0$ when $b_i = 1$ and $= 1$ when $b_i = 0$,
which means ${f_a}^{(\Tout)}(X_i)[-\Tout{:}] = R_S(X_i)$ if $b_i = 1$ and differs if $b_i = 0$.
So $S$ is trace shattered and $\TSdim(\mathcal{F}; \mathcal{X}, \Tout) \ge n$.
\end{proof}

\subsection{Upper Bound}

\begin{proof}[Proof of the upper bound in \Cref{thm:cot_fundamental}]
By \Cref{lem:tcs_equals_vc}, $\mathrm{VCdim}(\mathcal{L}_{\mathrm{CoT}}) = d_{\mathrm{TS}}$.
In the realizable setting, $\mathrm{Cons}_{\sCoT}$ can output some $\hat{f}$ with
zero empirical loss on the sample (i.e., $\ell_{\hat{f}}(X_i, Z_i) = 0$ for all $i$), where $Z_i = {f_*}^{(\Tout)}(X_i)[-\Tout{:}]$.

It remains to show that \emph{any} hypothesis in a VC class that is consistent on $m$ samples has small true error.
A standard VC generalization bound (realizable case) states that there exists a universal constant $C > 0$
such that if
\begin{equation}
m \ge C \cdot \frac{d_{\mathrm{TS}} \log(1/\varepsilon) + \log(1/\delta)}{\varepsilon},
\end{equation}
then with probability at least $1 - \delta$,
\begin{equation}
\mathbb{E}_{X\sim \mathcal{D}}\big[\ell_{\hat{f}}\big(X, {f_*}^{(\Tout)}(X)[-\Tout{:}]\big)\big] \le \varepsilon.
\end{equation}
Finally, \Cref{lem:e2e_le_cot} gives $\mathcal{L}_{\mathrm{e2e}}(\hat{f}) \le \varepsilon$.
\end{proof}

\subsection{Lower Bound}

\begin{proof}[Proof of the lower bound in \Cref{thm:cot_fundamental}]
Let $d := d_{\mathrm{AS}}$ and assume $d \ge 2$ (otherwise the claim is trivial).
By \Cref{def:tidim}, there exist a set $S = \{X_1, \ldots, X_d\} \subseteq \mathcal{X}$,
a prefix assignment $R_S: S \to \Sigma^{\Tout-1}$,
such that for every $b \in \{0,1\}^d$ there exists $f_b \in \mathcal{F}$ with
\begin{equation}
{f_b}^{(\Tout)}(X_i)[-\Tout{:}-1] = R_S(X_i),
\qquad
{f_b}^{(\Tout)}(X_i)[-1] = b_i
\quad \forall i \in [d],
\end{equation}
so the first $\Tout-1$ generated tokens are fixed while the final token equals $b_i$.

Fix $\varepsilon \in (0, 1/8)$.
We define a hard distribution $\mathcal{D}$ supported on $S$:
put a large mass on $X_1$ and spread the remaining mass over $X_2, \ldots, X_d$:
\begin{equation}
\mathcal{D}(X_1) = 1 - 4\varepsilon,
\qquad
\mathcal{D}(X_i) = \frac{4\varepsilon}{d-1} \quad \text{for } i = 2, \ldots, d.
\end{equation}

\paragraph{Random target choice.}
Let $B = (B_1, \ldots, B_d) \in \{0,1\}^d$ be uniformly random, and choose the corresponding
target generator $f_B \in \mathcal{F}$ guaranteed above.
We analyze the performance of an arbitrary learning rule $A$ under the joint randomness of
$B$ and the sample $S_{\sCoT} \sim \mathcal{D}^m$ labeled by full traces of $f_B$.

Let $\widehat g:=A(S_{\sCoT})\in\Sigma^{\mathcal X}$ be the final-token predictor returned by the learner.
Define the \emph{rare-points risk}
\begin{equation}
\mathcal{L}_{\mathrm{rare}} := \sum_{i=2}^d \mathcal{D}(X_i) \cdot \mathbf{1}\bigl[\widehat{g}(X_i) \neq B_i\bigr].
\end{equation}
Note that $0 \le \mathcal{L}_{\mathrm{rare}} \le \sum_{i=2}^d \mathcal{D}(X_i) = 4\varepsilon$.

\paragraph{Key observation: if a point is unseen, its label is unpredictable.}
Fix $i \in \{2, \ldots, d\}$ and let $E_i$ be the event that $X_i$ does not appear in the sample $S_{\sCoT}$.
On the event $E_i$, the entire observed sample (inputs and traces) is independent of the bit $B_i$:
indeed, the sample contains traces only for points $X_j$ that were drawn, and for every such point
$X_j \in S$ the trace prefix is the fixed value $R_S(X_j)$ (independent of $B$), while the final token
reveals only $B_j$ for those drawn indices $j$.
Formally, conditioned on $E_i$, the sample is a function of $\{B_j : X_j \text{ appears in } S_{\sCoT}\}$ and is therefore independent of $B_i$.
Since $X_i$ never appears under $E_i$, the bit $B_i$ is never revealed in the data and remains
uniform $\{0,1\}$ even conditioned on the realized sample.
Therefore,
\begin{equation}
\mathbb{P}\bigl[\widehat{g}(X_i) \neq B_i \mid E_i\bigr] \ge \frac{1}{2}.
\end{equation}
Taking expectation,
$\mathbb{P}\bigl[\widehat{g}(X_i) \neq B_i\bigr] \ge \mathbb{P}(E_i) \cdot \frac{1}{2}$.

\paragraph[Compute probability of $E_i$]{Compute $\mathbb{P}(E_i)$.}
Each draw hits $X_i$ with probability $4\varepsilon/(d-1)$, hence
$\mathbb{P}(E_i) = \bigl(1 - \frac{4\varepsilon}{d-1}\bigr)^m$.

\paragraph{Lower bound the expected risk.}
By linearity of expectation and symmetry over $i = 2, \ldots, d$,
\begin{equation}
\mathbb{E}[\mathcal{L}_{\mathrm{rare}}]
= \sum_{i=2}^d \mathcal{D}(X_i) \cdot \mathbb{P}\bigl[\widehat{g}(X_i) \neq B_i\bigr]
\ge \sum_{i=2}^d \frac{4\varepsilon}{d-1} \cdot \frac{1}{2} \cdot
\Bigl(1 - \frac{4\varepsilon}{d-1}\Bigr)^m
= 2\varepsilon \Bigl(1 - \frac{4\varepsilon}{d-1}\Bigr)^m.
\end{equation}
Using Bernoulli's inequality $(1-u)^m \ge 1 - mu$ for $u \in [0,1]$, if
$m \le \frac{d-1}{16\varepsilon}$,
then $(1 - \frac{4\varepsilon}{d-1})^m \ge \frac{3}{4}$, so
$\mathbb{E}[\mathcal{L}_{\mathrm{rare}}] \ge 2\varepsilon \cdot \frac{3}{4} = \frac{3}{2}\varepsilon$.
Since $d\ge 2$, the assumption $m < d/(32\varepsilon)$ implies $m \le (d-1)/(16\varepsilon)$, so this condition is satisfied after adjusting constants.

\paragraph{From expectation to probability.}
Let
\[
p := \mathbb{P}_{B,S_{\sCoT}}(\mathcal{L}_{\mathrm{rare}} \ge \varepsilon).
\]
Since $\mathcal{L}_{\mathrm{rare}} \le 4\varepsilon$, we have
$\mathbb{E}[\mathcal{L}_{\mathrm{rare}}] \le \varepsilon \cdot (1-p) + 4\varepsilon \cdot p = \varepsilon + 3\varepsilon p$.
Thus $\mathbb{E}[\mathcal{L}_{\mathrm{rare}}] \ge \frac{3}{2}\varepsilon$ implies $p \ge \frac{1}{6}$.
In other words,
$\mathbb{P}_{B, S_{\sCoT}}(\mathcal{L}_{\mathrm{rare}} \ge \varepsilon) \ge \frac{1}{6}$.

\paragraph{Fix a hard target.}
By averaging over $B$, there exists a particular $b^* \in \{0,1\}^d$ such that, letting $f_* := f_{b^*}$,
\begin{equation}
\mathbb{P}_{S_{\sCoT} \sim \mathcal{D}^m}\!\left(
\sum_{i=2}^d \mathcal{D}(X_i) \cdot \mathbf{1}\bigl[A(S_{\sCoT})(X_i) \neq {f_*}^{(\Tout)}(X_i)[-1]\bigr]
\ge \varepsilon
\right) \ge \frac{1}{6}.
\end{equation}
Since the full end-to-end risk dominates the rare-points risk $\mathcal{L}_{\mathrm{rare}}$, the same lower bound holds for
$\mathbb{P}\bigl(\mathbb{P}_{X \sim \mathcal{D}}[A(S_{\sCoT})(X) \neq {f_*}^{(\Tout)}(X)[-1]] \ge \varepsilon\bigr)$.
Setting $c_0 = 1/32$ (absorbing constants) and $c_1 = 1/6$ proves the theorem.
\end{proof}

%% file: cot_ub.tex
\section{Proof of CoT Sample Complexity Upper Bound (\texorpdfstring{\Cref{thm:cot_upper_bound}}{Theorem \ref{thm:cot_upper_bound}})}\label{app:cot-upper-bound-proof}

\begin{proof}[Proof of \Cref{thm:cot_upper_bound}]
Throughout this proof, we use the notation from \Cref{thm:cot_upper_bound}: $\Tout$ denotes the generation length, and $\Ttotal = \Tin + \Tout$ denotes the total context length.
Fix any $\mathcal{F}\in \Ffam_{L,W}$.

We prove the $\TSdim$ bound for this fixed class by upper bounding the VC-dimension of the CoT loss class.
Recall the CoT loss class:
\begin{equation}
\ell_f(X, Z) := \mathbf{1}\big[f^{(\Tout)}(X)[-\Tout{:}] \neq Z\big],
\qquad
\mathcal{L}_{\mathrm{CoT}} := \{\ell_f : f \in \mathcal{F}\}.
\end{equation}
Equivalently (since $f$ is deterministic), $\ell_f(X, Z) = 0$ iff $f(X \concat Z[{:}t]) = Z[t]$ for all $t \in \{0,\ldots,\Tout-1\}$.
Let $K := |\Sigma|$.
By \Cref{lem:tcs_equals_vc} (with $T \leftarrow \Tout$),
\begin{equation}\label{eq:tc-vc}
\TSdim(\mathcal{F}; \mathcal{X}, \Tout) = \mathrm{VCdim}(\mathcal{L}_{\mathrm{CoT}}).
\end{equation}
Hence it suffices to show $\mathrm{VCdim}(\mathcal{L}_{\mathrm{CoT}}) = \mathcal{O}(WL \log(\Ttotal WLK))$.

\paragraph[Step 1: Unroll teacher-forcing prefixes]{Step 1: Unroll teacher-forcing prefixes and reduce to a growth bound for $\mathcal{F}$.}
Fix $n \ge 1$ and consider any $n$ labeled examples $\{(X_i, Z_i)\}_{i=1}^n \subseteq \mathcal{X} \times \Sigma^{\Tout}$.
For each $(i, t) \in [n] \times \{0,\ldots,\Tout-1\}$, define the teacher-forcing prefix context
\begin{equation}
c_{i,t} := X_i \concat Z_i[{:}t].
\end{equation}
Each $c_{i,t}$ has length at most $\Tin + t \le \Ttotal-1$.

Let $M := n\Tout$ and enumerate these contexts as $\{c_j\}_{j=1}^M := \{c_{i,t}\}_{(i,t) \in [n] \times \{0,\ldots,\Tout-1\}}$.
Define the multi-class growth function for next-token generators:
\begin{equation}
\Pi_{\mathcal{F}}(c_{1:M}) := \left|\left\{(f(c_1), \ldots, f(c_M)) \in \Sigma^M : f \in \mathcal{F}\right\}\right|,
\qquad
\Pi_{\mathcal{F}}(M) := \max_{\substack{c_{1:M}\\ |c_m|\le \Ttotal\ \forall m}} \Pi_{\mathcal{F}}(c_{1:M}).
\end{equation}

Now observe: for a fixed $f$, the loss vector $(\ell_f(X_1, Z_1), \ldots, \ell_f(X_n, Z_n)) \in \{0,1\}^n$
is a deterministic function of the prediction vector $(f(c_1), \ldots, f(c_M)) \in \Sigma^M$,
since $\ell_f(X_i, Z_i) = 0$ iff $f(c_{i,t}) = Z_i[t]$ for all $t \in \{0,\ldots,\Tout-1\}$.

Therefore, the number of distinct $\{0,1\}^n$ labelings realized by $\mathcal{L}_{\mathrm{CoT}}$ on $\{(X_i, Z_i)\}_{i=1}^n$ is at most $\Pi_{\mathcal{F}}(c_{1:M}) \le \Pi_{\mathcal{F}}(M)$.
In particular, if $\mathcal{L}_{\mathrm{CoT}}$ shatters these $n$ examples, then $2^n$ distinct loss vectors occur, so
\begin{equation}\label{eq:key-ineq}
2^n \le \Pi_{\mathcal{F}}(n\Tout).
\end{equation}

\paragraph{Step 2: A growth bound for multi-class Transformer next-token generators.}
We claim there exist universal constants $C, C' > 0$ such that for all $M \ge 1$,
\begin{equation}\label{eq:PiF-bound}
\log \Pi_{\mathcal{F}}(M) \le C \cdot WL \cdot \log(C' M\Ttotal WLK).
\end{equation}

\begin{proof}[Proof of \eqref{eq:PiF-bound}]
		Fix arbitrary contexts $c_{1:M}$, each of length at most $\Ttotal$, and write $\tau_m:=|c_m|-1$ for the last position of context $c_m$. Consider the parameter vector $\theta \in \mathbb{R}^W$.
		We apply the same recursive partitioning construction as in the VC upper bound proof (\Cref{app:upper-bound-proof}), with $N \leftarrow M$ and context length $\Ttotal$.
		Concretely, at each layer we first refine by the attention score-difference polynomials (freezing Average-Hard routing/ties), and then refine by all internal $\mathsf{FFN}$ ReLU pre-activations (freezing ReLU masks), exactly as in \Cref{app:upper-bound-proof}.
	Note that since the token embedding matrix $W_{\mathsf{TE}}$ is included in $\theta$, the layer-$0$ representations are affine in $\theta$; this is already accounted for in the base case of \Cref{app:upper-bound-proof}.
	Thus we obtain a final partition $\mathcal{P}_L$ of $\mathbb{R}^W$ such that:
\begin{itemize}
    \item[(i)] For every cell $C \in \mathcal{P}_L$, all branching decisions are fixed on $c_{1:M}$;
    \item[(ii)] On each $C \in \mathcal{P}_L$, every hidden representation $H^{(L)}_{m,t}(\theta)$ is polynomial in $\theta$ with degree $\deg_{\mathrm{out}} = \mathcal{O}(L)$;
    \item[(iii)] The number of cells satisfies (by the same Warren bound calculation as the VC upper bound proof):
    \begin{equation}\label{eq:cells}
    \log |\mathcal{P}_L| \le \mathcal{O}(WL \log(M\Ttotal WL)).
    \end{equation}
\end{itemize}

It remains to bound, for a fixed final cell $C \in \mathcal{P}_L$, how many distinct next-token prediction vectors $(f_\theta(c_1), \ldots, f_\theta(c_M))$ can arise as $\theta$ ranges over $C$.

For each input $m \in [M]$ and token $x \in \Sigma$, define the output logit
\begin{equation}
z_{m,x}(\theta) := \big\langle W_{\mathsf{DE}}[x,:], H^{(L)}_{m, \tau_m}(\theta) \big\rangle.
\end{equation}
On a fixed cell $C$, $H^{(L)}_{m, \tau_m}(\theta)$ is polynomial in $\theta$ of degree $\mathcal{O}(L)$, so the logit $z_{m,x}(\theta)$ is polynomial in $\theta$ of degree at most $\mathcal{O}(L)+1 = \mathcal{O}(L)$ (since the decoding matrix $W_{\mathsf{DE}}$ is a coordinate projection of $\theta$).
The predicted token is $f_\theta(c_m) = \arg\max_{x \in \Sigma} z_{m,x}(\theta)$, with any fixed tie-breaking rule.

For multi-class argmax outputs, it suffices to control the ternary sign pattern of all pairwise logit differences:
\begin{equation}
\Delta_{m,x,x'}(\theta) := z_{m,x}(\theta) - z_{m,x'}(\theta), \qquad m \in [M],\ x < x' \in \Sigma.
\end{equation}
There are $m_{\mathrm{out}} = M \binom{K}{2} = \mathcal{O}(MK^2)$ such polynomials, each of degree at most $D_{\mathrm{out}} = \mathcal{O}(L)$.
On any region where all these ternary signs are fixed, the induced ordering (with ties) of the logits is fixed for every $m$, hence the argmax output $f_\theta(c_m)$ is fixed as well.
Therefore, the number of distinct prediction vectors realized within $C$ is at most the number of distinct ternary sign patterns of $\{\Delta_{m,x,x'}\}$.

Applying Warren's bound (\Cref{lem:warren-app}) to these $m_{\mathrm{out}}$ polynomials (padding with constant polynomials if $m_{\mathrm{out}} < W$) gives
\begin{equation}\label{eq:within-cell}
\log \Pi_C \le \mathcal{O}(W \log(MKL)),
\end{equation}
where $\Pi_C$ denotes the number of distinct prediction vectors realized by parameters in cell $C$.

Finally, summing over cells,
\begin{equation}
\Pi_{\mathcal{F}}(c_{1:M}) \le \sum_{C \in \mathcal{P}_L} \Pi_C \le |\mathcal{P}_L| \cdot \max_{C \in \mathcal{P}_L} \Pi_C.
\end{equation}
Combining \eqref{eq:cells} and \eqref{eq:within-cell}, absorbing the within-cell term into $O(WL\log(M\Ttotal WLK))$, yields \eqref{eq:PiF-bound}.
\end{proof}

\paragraph[Step 3: Conclude the VC bound]{Step 3: Conclude the VC bound for $\mathcal{L}_{\mathrm{CoT}}$.}
Combine \eqref{eq:key-ineq} with \eqref{eq:PiF-bound} at $M = n\Tout$:
\begin{equation}
n \le \log \Pi_{\mathcal{F}}(n\Tout) \le C \cdot WL \cdot \log(C' n\Tout \Ttotal WLK).
\end{equation}
Using $\log(n\Tout \Ttotal WLK) \le \log n + \log(\Tout \Ttotal WLK)$ and $\Tout\le \Ttotal$, we get
\begin{equation}
n \le C \cdot WL \cdot \log n + C \cdot WL \cdot \log(\Ttotal WLK) + \mathcal{O}(WL).
\end{equation}
Applying the standard inversion (if $n \le a \log n + b$ then $n = \mathcal{O}(a \log a + b)$), with $a = \Theta(WL)$ and $b = \Theta(WL \log(\Ttotal WLK))$, yields
\begin{equation}
\mathrm{VCdim}(\mathcal{L}_{\mathrm{CoT}}) = \mathcal{O}(WL \log(\Ttotal WLK)).
\end{equation}
When $K\le \mathrm{poly}(W)$, this simplifies to $\mathcal{O}(WL \log(\Ttotal W))$.
By \eqref{eq:tc-vc}, this proves the $\TSdim$ bound.

\paragraph{Step 4: Consequence for sample complexity.}
Plugging $d_{\mathrm{TS}} = \TSdim(\mathcal{F}; \mathcal{X}, \Tout)$ into \Cref{thm:cot_fundamental} gives
\begin{equation}
m(\varepsilon, \delta) = O\!\left(\frac{d_{\mathrm{TS}} \log(1/\varepsilon) + \log(1/\delta)}{\varepsilon}\right),
\end{equation}
and substituting $d_{\mathrm{TS}} = \mathcal{O}(WL \log(\Ttotal WLK))$ yields the corresponding sample-complexity upper bound. Since $\mathcal{F}\in\Ffam_{L,W}$ was arbitrary, the $\TSdim$ bound holds for every class in the family.
\end{proof}

%% file: cot_lb.tex
\section{Proof of CoT Sample Complexity Lower Bound}\label{app:cot-lower-bound-proof}

In this appendix, we prove the lower bounds for CoT sample complexity. We first establish that the answer shattering dimension can grow logarithmically with the generation length (\Cref{thm:cot_shattering}), then prove the Transformer scratchpad $\ASdim$ lower bound and combine it with the ordinary supervised obstruction (\Cref{cor:cot-tidim-lower}).

\subsection{Logarithmic Shattering via Scratchpad (\texorpdfstring{\Cref{thm:cot_shattering}}{Theorem \ref{thm:cot_shattering}})}

\begin{proof}[Proof of \Cref{thm:cot_shattering}]
We construct a constant-depth, constant-width autoregressive Transformer whose first
$\Tout-1$ generated tokens are independent of the labeling, while the final token
realizes an arbitrary labeling of $\Omega(\log \Ttotal)$ inputs.

\paragraph{Next-position convention.}
At the step that generates the token in absolute position $\tau$, the visible prefix
occupies positions $0,\ldots,\tau-1$, and the Transformer output is read from the
last visible position $\tau-1$, as in \Cref{subsec:autoregressive}. Since the
positional encoding contains the scalar position, the representation at the last
visible position can compute the next-position scalar $\tau=(\tau-1)+1$. Below,
when we say that the current step uses $\tau$, we mean this computed next-position
scalar; no extra query token is introduced.

Assume first that $N$ is larger than a universal constant; smaller $N$ are absorbed
by constants. Let
\begin{equation}
m:=\lfloor \log_2 N\rfloor,\qquad M:=2^m,
\end{equation}
so $M=\Theta(N)$. We answer-shatter $m$ inputs. Set
\begin{equation}
\ell:=\lceil \log_2 m\rceil,\qquad L_{\mathrm{in}}:=\ell+1,\qquad
L_{\mathrm{out}}:=m+1,\qquad \Tin:=L_{\mathrm{in}}.
\end{equation}

\paragraph{Decorated constant-size alphabet.}
The main text displays the scratchpad using the symbols
$\{0,1,\#,\mathrm{END0},\mathrm{END1}\}$. Formally, we use constantly many
decorated copies of these tokens. These are still ordinary tokens in a
constant-size alphabet; the decorations simply store the finite-state
information, such as carry and borrow bits, in the generated prefix itself. Let
\begin{equation}
\Sigma =
\{0,1,\mathrm{END0},\mathrm{END1},\mathrm{PAD}\}
\cup \Gamma_{\mathrm{sep}}\cup \Gamma_{\mathrm{dec}}\cup \Gamma_{\mathrm{tab}},
\end{equation}
where
\begin{equation}
\Gamma_{\mathrm{sep}}
=\{\#_{\mathrm{in}}^{z},\#_{\mathrm{dec}}^{z}:z\in\{0,1\}\}\cup\{\#_{\mathrm{tab}}\},
\end{equation}
\begin{equation}
\Gamma_{\mathrm{dec}}=\{D_{b,\rho,z}: b,\rho,z\in\{0,1\}\},\qquad
\Gamma_{\mathrm{tab}}=\{T_{b,\rho}: b,\rho\in\{0,1\}\}.
\end{equation}
Here $b$ is the visible bit. In $D_{b,\rho,z}$, $\rho$ is the outgoing borrow
for the decrement rule and $z$ records whether all bits generated so far in the
current decode row are zero. In $T_{b,\rho}$, $\rho$ is the outgoing carry for
the increment rule. The superscript on $\#_{\mathrm{in}}^z$ and
$\#_{\mathrm{dec}}^z$ records whether the preceding row is all zero. These
decorations are deterministic functions of the input and the already generated
prefix, and are never label-dependent.

For each integer $a$, let $\mathrm{bin}_r(a)$ be the $r$-bit binary representation
of $a$ in least-significant-bit first order. Define
\begin{equation}
x_i:=\mathrm{bin}_{\ell}(i)\#_{\mathrm{in}}^{\mathbf{1}[i=0]}\in\Sigma^{\Tin},
\qquad i=0,\ldots,m-1,
\end{equation}
and let $S:=\{x_0,\ldots,x_{m-1}\}$.

\paragraph{Addressing and finite-state primitives.}
We use fixed quadratic positional features: the token at absolute position $p$
contains positional coordinates $(p,-p^2,1)$. Thus a hard-attention head can
retrieve any visible integer position $r$: with query $q(r)=(2r,1,0)$ and keys
$k_p=(p,-p^2,1)$, the score is
\begin{equation}
\langle q(r),k_p\rangle = 2rp-p^2 = r^2-(p-r)^2,
\end{equation}
whose unique visible maximizer is $p=r$.

We also use constant-size ReLU equality gates. For integer $a$, define
\begin{equation}
\mathrm{eq}_a(u):=
2\left(
\mathrm{ReLU}\!\left(u-a+\frac12\right)
-2\mathrm{ReLU}(u-a)
+\mathrm{ReLU}\!\left(u-a-\frac12\right)
\right).
\end{equation}
For every integer $u$, $\mathrm{eq}_a(u)=\mathbf{1}[u=a]$. Any Boolean function
of constantly many token-type indicators, state bits, and such equality gates
can be implemented by a constant-size ReLU FFN. All constants such as $m,M$ and
the final position are fixed as part of the architecture for this value of $N$.

\paragraph{The label-invariant scratchpad.}
Fix an arbitrary labeling $y=(y_0,\ldots,y_{m-1})\in\{0,1\}^m$ and encode it as
\begin{equation}
s=s(y):=\sum_{i=0}^{m-1} y_i2^i\in\{0,\ldots,M-1\},
\qquad\text{so that}\qquad \mathrm{bit}_i(s)=y_i.
\end{equation}
The only label-dependent trainable scalar is
\begin{equation}
\beta_s:=sL_{\mathrm{out}}.
\end{equation}
It is gated off until the final generation step. Set
\begin{equation}
\Tout:=mL_{\mathrm{in}}+ML_{\mathrm{out}}+3.
\end{equation}
For input $x_i$, the first $\Tout-1$ generated tokens form a prefix
$R_S(x_i)$ independent of $s$. If decorations are ignored, this prefix is
\begin{equation}
\mathrm{bin}_{\ell}(i-1)\#,\ \mathrm{bin}_{\ell}(i-2)\#,\ \ldots,\ 
\mathrm{bin}_{\ell}(0)\#,\ \mathrm{END0}
\end{equation}
(empty before $\mathrm{END0}$ when $i=0$), followed by
\begin{equation}
\mathrm{bin}_{m}(0)\#,\ \mathrm{bin}_{m}(1)\#,\ \ldots,\ 
\mathrm{bin}_{m}(M-1)\#,\ \mathrm{END1},
\end{equation}
followed by $\mathrm{PAD}$ tokens until the prefix length is exactly $\Tout-1$.

\paragraph{Generating the scratchpad.}
We describe the next-token rule. A constant number of heads retrieve: (i) the
latest boundary token among the decorated separators, $\mathrm{END0}$, and
$\mathrm{END1}$, using a score with a large type bonus plus the position
coordinate; (ii) the unique $\mathrm{END0}$ token if it has appeared; (iii) the unique
$\mathrm{END1}$ token if it has appeared; (iv) the same-column predecessor at
position $\tau-L_{\mathrm{in}}$; and (v) the same-column predecessor at position
$\tau-L_{\mathrm{out}}$. The quadratic addressing primitive implements the
constant-offset predecessor retrievals. When such a predecessor position is not
visible, the corresponding value is ignored by the FFN.

If $\mathrm{END1}$ has already appeared and the current step is not the final
step, the model outputs $\mathrm{PAD}$. If $\mathrm{END0}$ has appeared but
$\mathrm{END1}$ has not, the model is in the table phase. Let $p_0$ be the
position of $\mathrm{END0}$. The table ends when
\begin{equation}
\mathrm{eq}_{ML_{\mathrm{out}}}\bigl(\tau-(p_0+1)\bigr)=1,
\end{equation}
at which point the model outputs $\mathrm{END1}$. Otherwise, write
$k=\tau-b-1$, where $b$ is the latest boundary position. If $k=m$, the model
outputs $\#_{\mathrm{tab}}$. If $0\le k<m$, it outputs the next bit of the
current table row. The first row after $\mathrm{END0}$ is all zeros. Later rows
are obtained by incrementing the previous row by one in LSB-first order:
\begin{equation}
v_k=q_k\oplus c_k,\qquad c_{k+1}=q_k\wedge c_k,
\end{equation}
where $q_k$ is the bit retrieved from position $\tau-L_{\mathrm{out}}$, and
$c_k$ is $1$ for $k=0$ and otherwise the carry bit stored in the previous
generated table token. The emitted token is $T_{v_k,c_{k+1}}$.

If $\mathrm{END0}$ has not appeared, the model is in the decode phase. If the
latest separator is $\#_{\mathrm{in}}^1$ or $\#_{\mathrm{dec}}^1$, then the
preceding row is all zero and the model outputs $\mathrm{END0}$. Otherwise let
$k=\tau-b-1$. If $k=\ell$, the model has completed a decode row; it reads the
zero-so-far flag $z$ from the previous generated decode token and outputs
$\#_{\mathrm{dec}}^z$. If $0\le k<\ell$, the model decrements the previous row
by one in LSB-first order:
\begin{equation}
v_k=q_k\oplus \rho_k,\qquad
\rho_{k+1}=(1-q_k)\wedge\rho_k,
\end{equation}
where $q_k$ is the bit retrieved from position $\tau-L_{\mathrm{in}}$, and
$\rho_k$ is $1$ for $k=0$ and otherwise the borrow bit stored in the previous
generated decode token. The zero-so-far flag is updated as
\begin{equation}
z_{k+1}=
\begin{cases}
\mathbf{1}[v_0=0], & k=0,\\
z_k\wedge \mathbf{1}[v_k=0], & k>0.
\end{cases}
\end{equation}
The emitted token is $D_{v_k,\rho_{k+1},z_{k+1}}$.

All these rules are Boolean functions of constantly many retrieved symbols,
state bits, and equality gates, and hence are implemented exactly by a
constant-size FFN. They generate the desired decode segment, then the universal
table, then padding. In particular, $\mathrm{END0}$ appears at absolute position
\begin{equation}
p_0(i)=(i+1)L_{\mathrm{in}}.
\end{equation}
The entire generated prefix is independent of $s$.

\paragraph{Gating the label-dependent scalar.}
Let
\begin{equation}
\tau_\star:=\Tin+\Tout-1
\end{equation}
be the absolute position of the final generated token, and set
$G_{\mathrm{fin}}:=\mathrm{eq}_{\tau_\star}(\tau)$. Since
$0\le \beta_s\le ML_{\mathrm{out}}$, choose a hard-wired constant
$B>ML_{\mathrm{out}}$. The product of the label-dependent scalar with the
final-step gate is implemented by
\begin{equation}
\beta_s^{\mathrm{gated}}
:=\mathrm{ReLU}\bigl(\beta_s-B(1-G_{\mathrm{fin}})\bigr).
\end{equation}
This equals $0$ before the final step and equals $\beta_s$ at the final step.
The output logits are gated similarly, so the scratchpad rule is used before
the final step and the readout rule below is used at the final step.

\paragraph{Final readout.}
At the final step, the visible prefix contains $\mathrm{END0}$, the full table,
$\mathrm{END1}$, and padding. The readout uses two consecutive attention
sublayers. The first attention sublayer retrieves the unique $\mathrm{END0}$
token and writes its position $v=p_0(i)$ into a work coordinate. The following
FFN computes
\begin{equation}
r=\left(1+\frac{1}{L_{\mathrm{in}}}\right)v+\beta_s^{\mathrm{gated}}.
\end{equation}
At the final step, $\beta_s^{\mathrm{gated}}=\beta_s=sL_{\mathrm{out}}$. Since
$v=p_0(i)=(i+1)L_{\mathrm{in}}$,
\begin{equation}
r=p_0(i)+i+1+sL_{\mathrm{out}}.
\end{equation}
The table starts at position $p_0(i)+1$, and row $s$ starts at
$p_0(i)+1+sL_{\mathrm{out}}$. Thus the position of the $i$-th bit of row $s$ is
\begin{equation}
p^\star=p_0(i)+1+sL_{\mathrm{out}}+i=r.
\end{equation}
The second attention sublayer uses query $q(r)=(2r,1,0)$ and retrieves the token
at $p^\star$. Its visible bit is $\mathrm{bit}_i(s)=y_i$, and the output layer
maps visible bit $0/1$ to the answer token $0/1$.

Therefore, for every labeling $y\in\{0,1\}^m$, choosing
$\beta_s=sL_{\mathrm{out}}$ gives a generator $f_s$ such that, for every
$i\in\{0,\ldots,m-1\}$,
\begin{equation}
f_s^{(\Tout)}(x_i)[-\Tout{:}-1]=R_S(x_i),
\qquad
f_s^{(\Tout)}(x_i)[-1]=y_i.
\end{equation}
Hence $S$ is answer shattered, and
\begin{equation}
\ASdim(\mathcal F;\Sigma^{\Tin},\Tout)\ge |S|=m.
\end{equation}

\paragraph{Size and length.}
The construction uses a constant number of heads, a constant number of layers,
and constant hidden dimension: the only work coordinates are constantly many
positions, Boolean flags, and token-type indicators. Residual connections or
identity sublayers copy these work coordinates between the two readout hops.

Finally,
\begin{equation}
\Tin=L_{\mathrm{in}}=O(\log m)=O(\log\log N),
\end{equation}
and
\begin{equation}
\Tout=mL_{\mathrm{in}}+ML_{\mathrm{out}}+3=\Theta(Mm)=\Theta(N\log N).
\end{equation}
Thus
\begin{equation}
\Ttotal=\Tin+\Tout=\Theta(\Tout)=\Theta(N\log N),
\end{equation}
and so $m=\Theta(\log N)=\Omega(\log \Ttotal)$. This proves the theorem.
\end{proof}

\subsection{CoT Lower Bounds for Transformers (Corollary~\ref{cor:cot-tidim-lower})}\label{app:cot-tidim-lower-proof}

\begin{proof}[Proof of \Cref{cor:cot-tidim-lower}]
We prove the scratchpad amplification lower bound.

\paragraph{Scratchpad amplification.}
Let $C_{\mathrm{amp}}$ be a sufficiently large universal constant and set
\begin{equation}
R:=\left\lfloor \frac{L}{C_{\mathrm{amp}}}\right\rfloor .
\end{equation}
The constant $C_{\mathrm{amp}}$ accounts for one recursive retrieval round, the constant-depth scratchpad transducer of \Cref{thm:cot_shattering}, and the constant-depth final readout. Choose $L_0\ge 2C_{\mathrm{amp}}$; since $L\ge L_0$, we have $R\ge1$.

Choose $m$ as the largest integer such that, with $M:=2^m$,
\begin{equation}
M(m+1)\le \Tout/C
\end{equation}
for a sufficiently large universal constant $C$. Then $m=\Theta(\log\Tout)$ and $M\le\Tout$. Let
\begin{equation}
\ell:=\lceil\log_2 m\rceil,\qquad L_{\mathrm{in}}:=\ell+1,\qquad L_{\mathrm{out}}:=m+1.
\end{equation}

\emph{Alphabet and inputs.}
Let $\Gamma$ be the constant-size decorated alphabet used in the proof of \Cref{thm:cot_shattering}. Enlarge it by identifier and round tokens:
\begin{equation}
\Sigma:=\Gamma\cup\{\mathtt{ID}_1,\ldots,\mathtt{ID}_n\}
\cup\{\mathtt{LYR}_1,\ldots,\mathtt{LYR}_R\}.
\end{equation}
Here $n$ will be chosen as $\Theta(W)$, so $|\Sigma|=O(W+L)$. For each triple $(j,k,i)\in[n]\times[R]\times\{0,\ldots,m-1\}$, define
\begin{equation}
X_{j,k,i}:=
\mathtt{ID}_j\concat \mathtt{LYR}_k\concat
\mathrm{bin}_{\ell}(i)\concat \#_{\mathrm{in}}^{\mathbf 1[i=0]}
\in\Sigma^{\Tin},
\qquad
\Tin:=2+L_{\mathrm{in}}.
\end{equation}
Let $S:=\{X_{j,k,i}\}$ and $\mathcal X:=S$. Then $|S|=nRm$.

\emph{Common prefix.}
For each $X_{j,k,i}$, the generator runs the scratchpad construction of \Cref{thm:cot_shattering} on the suffix $\mathrm{bin}_{\ell}(i)\concat \#_{\mathrm{in}}^{\mathbf 1[i=0]}$ and ignores the two-token header while generating the first $\Tout-1$ tokens. The header only shifts absolute scratchpad positions by
\begin{equation}
H:=2.
\end{equation}
Thus all parameter settings share the same prefix assignment $R_S:S\to\Sigma^{\Tout-1}$: the DECODE segment, the universal TABLE segment, and PAD tokens, exactly as in \Cref{thm:cot_shattering}. This prefix depends on $i$ but not on the labeling.

\emph{Encoding labels.}
Fix any labeling $b\in\{0,1\}^{S}$, written as $b_{j,k,i}$. For each $(j,k)\in[n]\times[R]$, encode the $m$ labels into
\begin{equation}
s_{j,k}:=\sum_{i=0}^{m-1}b_{j,k,i}2^i\in\{0,\ldots,M-1\},
\qquad
\mathrm{bit}_i(s_{j,k})=b_{j,k,i}.
\end{equation}

\emph{Multiplexing row indices.}
The embedding of $\mathtt{ID}_j$ stores the packed scalar
\begin{equation}\label{eq:cot-packed-row-indices}
r_j^{(0)}:=\sum_{u=1}^{R}\frac{s_{j,u}}{(2\Tout)^{u-1}},
\end{equation}
and the embedding of $\mathtt{LYR}_k$ stores $g:=k$. At the final generation step, a constant number of heads copies $(r_j^{(0)},g)$ from the two header tokens into protected work coordinates. These coordinates are ignored before the final step, so they cannot affect $R_S$.

The recursive retrieval then proceeds for $R$ rounds. At the beginning of round $u$, the register has the form
\begin{equation}
r^{(u-1)}
=
\sum_{v=u}^{R}\frac{s_{j,v}}{(2\Tout)^{v-u}}
=s_{j,u}+\delta_u,
\qquad 0\le\delta_u<\frac12 .
\end{equation}
Using quadratic positional features, a hard-attention head with query $q=(2r^{(u-1)},1,0)$ selects the unique visible position $t^\star=s_{j,u}$. The $\mathsf{FFN}$ performs the base-$2\Tout$ left shift
\begin{equation}
r^{(u)}:=2\Tout(r^{(u-1)}-t^\star).
\end{equation}
To keep only the selected digit, use the usual ReLU hat gate $\mathrm{gate}_u(g)=\mathbf 1[g=u]$ on integer $g\in[R]$ and update a protected accumulator by
\begin{equation}
\hat s\leftarrow \hat s+
\Bigl(
\mathrm{ReLU}(t^\star+M_0(\mathrm{gate}_u(g)-1))
-
\mathrm{ReLU}(M_0(\mathrm{gate}_u(g)-1))
\Bigr),
\end{equation}
where $M_0>\Tout$ is fixed. This adds $t^\star$ iff $u=k$ and adds zero otherwise. Hence after $R$ rounds, $\hat s=s_{j,k}$. Set $\beta:=L_{\mathrm{out}}\hat s$.

\emph{Final readout.}
At the final step, the prefix contains the complete scratchpad. Let $v$ be the absolute position of the unique $\mathrm{END0}$ token. Because of the header,
\begin{equation}
v=H+(i+1)L_{\mathrm{in}}.
\end{equation}
A first readout attention sublayer retrieves this token and writes $v$ into a work coordinate. The next $\mathsf{FFN}$ computes
\begin{equation}\label{eq:cot-amplified-target-address}
r:=v+\beta+\frac{v-H}{L_{\mathrm{in}}}
=H+(i+1)L_{\mathrm{in}}+s_{j,k}L_{\mathrm{out}}+(i+1).
\end{equation}
This is exactly the absolute position of the $i$-th bit in row $s_{j,k}$ of the TABLE segment. A second readout attention sublayer retrieves that token, whose visible bit is $\mathrm{bit}_i(s_{j,k})=b_{j,k,i}$. The output layer maps visible bit $0/1$ to two answer tokens.

\emph{Shattering and parameter budget.}
The scratchpad subspace never reads the multiplexing registers, and the output rule for the first $\Tout-1$ generated tokens depends only on the scratchpad subspace. Therefore every labeling produces the same prefix assignment $R_S$, and only the final token depends on $s_{j,k}$. Thus $S$ is answer shattered and
\begin{equation}
\ASdim(\mathcal F_\star;\mathcal X,\Tout)
\ge |S|
=nRm
=\Omega(WL\log\Ttotal),
\end{equation}
where $m=\Theta(\log\Tout)=\Theta(\log\Ttotal)$ because $\Tin=O(\log\log\Tout)$. The labeling-dependent parameters are the $n$ packed ID embeddings, costing $O(n)$ parameters; the round-token embeddings and backbone cost $O(R)\subseteq O(L)$. Choosing $n=\Theta(W)$ and $W\ge C_0L$ gives total parameter count at most $W$.

\paragraph{Sample-complexity lower bound.}
\emph{Scratchpad witness.}
The $\ASdim$ bound above and \Cref{thm:cot_fundamental} give
\begin{equation}
\Omega\left(\frac{WL\log\Ttotal}{\varepsilon}\right)
\end{equation}
samples.

\emph{Supervised ReLU witness via token prompts.}
We use the same fixed-prefix reduction, but realize the supervised inputs as token prompts rather than embedding-valued prompts. By the ReLU subclass from \Cref{lem:relu-subclass-app}, after adjusting constants the Transformer family contains a ReLU-network subclass $\mathcal H$ with
\begin{equation}
\mathrm{VCdim}(\mathcal H)=\Omega(WL\log(W/L)).
\end{equation}
Let $D:=\lfloor c_{\mathrm{R}}WL\log(W/L)\rfloor$ and choose shattered points $a_1,\ldots,a_D$ for $\mathcal H$. The ``sufficiently large $\Tout$'' condition includes $\Tout\ge C_{\mathrm{R}}D$, so the prompt length below is at most a constant fraction of the total context length.

Let $\Sigma_{\mathrm{R}}=\{\mathtt{blank},\mathtt{mark},\mathtt{query},\mathtt{\#},0,1\}$. For each $i\in[D]$, define a token prompt $X_i\in\Sigma_{\mathrm{R}}^{D+1}$ with $\mathtt{mark}$ at position $i$, $\mathtt{query}$ at the last prompt position, and $\mathtt{blank}$ elsewhere. The positional encoding is fixed and non-learned; we set its first coordinates at position $i$ to be the ReLU input $a_i$. Thus the inputs $a_i$ are supplied by the fixed positional encoding and do not consume the $W$ trainable parameters.

At the final generation step, a constant-size attention wrapper retrieves the unique $\mathtt{mark}$ token and copies these positional-encoding coordinates to the current output position. The remaining layers simulate the ReLU network $h\in\mathcal H$ on the copied vector. Before the final step, a positional gate forces the output to be the fixed dummy token $\mathtt{\#}$; at the final step, the dummy branch is disabled and the classifier branch outputs $0$ or $1$. Hence each $h\in\mathcal H$ gives a token-prompt generator $f_h$ satisfying
\begin{equation}
f_h^{(\Tout)}(X_i)[-\Tout{:}-1]=\mathtt{\#}^{\Tout-1},
\qquad
f_h^{(\Tout)}(X_i)[-1]=h(a_i).
\end{equation}
The prefix is independent of $h$, so any CoT learner for this token-prompt subclass is also a supervised learner for $\mathcal H$: given labeled examples $(a_i,h(a_i))$, feed the learner the CoT examples $(X_i,\mathtt{\#}^{\Tout-1}\concat h(a_i))$ and read off its final-token predictor. The VC lower bound in \Cref{thm:fundamental-vc} gives
\begin{equation}
\Omega\left(\frac{WL\log(W/L)}{\varepsilon}\right)
\end{equation}
samples.

Taking the harder of these two witnesses,
\begin{equation}
\max\left\{
\Omega\left(\frac{WL\log\Ttotal}{\varepsilon}\right),
\Omega\left(\frac{WL\log(W/L)}{\varepsilon}\right)
\right\}
=
\Omega\left(\frac{WL\log(\Ttotal W/L)}{\varepsilon}\right),
\end{equation}
after adjusting constants.
\end{proof}